\documentclass[twoside]{article}

\usepackage[accepted]{aistats2021}
\usepackage[round]{natbib}

\usepackage[utf8]{inputenc} 
\usepackage[T1]{fontenc}    
\usepackage{hyperref}       
\usepackage{url}            
\usepackage{booktabs}       
\usepackage{amsfonts}       
\usepackage{nicefrac}       
\usepackage{microtype}      
\usepackage{lipsum}
\usepackage[subtle]{savetrees}
\usepackage{graphicx}
\usepackage{nicefrac}       
\usepackage{microtype}      
\usepackage{multirow}
\usepackage{amsfonts}       
\usepackage{amssymb}
\usepackage{algorithmic}
\usepackage{algorithm}
\usepackage{amsthm}
\usepackage{parallel}
\usepackage[round]{natbib}

\usepackage{wrapfig}
\usepackage{amsmath}
\usepackage{mathtools}
\usepackage{subcaption}
\usepackage{comment, color}
\usepackage{multicol, blindtext}
\usepackage{float}

\DeclareMathOperator*{\argmax}{arg\,max}
\DeclareMathOperator*{\argmin}{arg\,min}
\newcommand{\indep}{\rotatebox[origin=c]{90}{$\models$}}
\newtheorem{theorem}{Theorem}[section]

\newtheorem{lemma}[theorem]{Lemma}

\newtheorem{definition}[theorem]{Definition}

\begin{document}

%

%

\twocolumn[

\aistatstitle{Competing AI: How does competition feedback affect machine learning?}

\aistatsauthor{ Antonio A. Ginart \And Eva Zhang  \And  Yongchan Kwon \And James Zou}
\aistatsauthor{\texttt{\{tginart,evazhang,yckwon,jamesz\}@stanford.edu}}
\aistatsaddress{ Stanford University, Palo Alto, CA } ]

\begin{abstract}
\vspace{-12pt}
This papers studies how competition affects machine learning (ML) predictors. As ML becomes more ubiquitous, it is often deployed by companies to compete over customers. For example, digital platforms like Yelp use ML to predict user preference and make recommendations. A service that is more often queried by users, perhaps because it more accurately anticipates user preferences, is also more likely to obtain additional user data (e.g. in the form of a Yelp review). Thus, competing predictors cause feedback loops whereby a predictor's performance impacts what training data it receives and biases its predictions over time. We introduce a flexible model of competing ML predictors that enables both rapid experimentation and theoretical tractability. We show with empirical and mathematical analysis that competition causes predictors to specialize for specific sub-populations at the cost of worse performance over the general population. We further analyze the impact of predictor specialization on the overall prediction quality experienced by users. We show that having too few or too many competing predictors in a market can hurt the overall prediction quality. Our theory is complemented by experiments  on several real datasets using popular learning algorithms, such as neural networks and nearest neighbor methods.
\vspace{-11pt}
\end{abstract}

\section{Introduction}

This paper studies what happens when machine learning (ML) predictors compete against each other. ML systems are deployed in ever more ubiquitous applications ranging from commerce to healthcare. It is becoming increasingly common for competing companies in similar markets to use ML to improve their services and attract customers or users. For example, platforms like Yelp\footnote{\tiny https://blog.yelp.com/2019/08/yelp-is-releasing-a-new-personalized-app-experience} and Tripadvisor\footnote{\tiny https://www.tripadvisor.com/engineering/personalized-recommendations-for-experiences-using-deep-learning/} both use ML to predict user preferences and make personalized recommendations for restaurants and other experiences. A user is more likely to use Yelp over Tripadvisor if they believe Yelp will give them a better recommendation than Tripadvisor (and vice-versa). Many users leave reviews, likes, or other forms of engagement on the platform that they end up using. Finally, the platform can use this feedback as new data to improve their predictive algorithms. The catch is that this form of user data is not an unbiased sample from the general population of users. Rather, it is biased by the fact that users that leave Yelp reviews are more likely to use Yelp more than, say, Tripadvisor. 


Competing ML predictors can emerge in diverse settings. Competing search engines predict the most relevant web links given a user's search query. Competing lenders use their ML predictors to assess client credit and offer loan packages. In the ML-as-a-service industry, companies routinely compete to sell their ML algorithms to clients. While the details of the competition vary across settings, a key characteristic is that competition generates temporal dynamics and feedback loops for the learning algorithms. A predictor's performance at one time instance could impact the training data it (or its competitor) observes. Training sets are no longer independent samples from the general population distribution (this is the statistical definition of sampling bias). In turn, this affects the performance and bias of the predictor over time. 

In this paper, we propose a model of competing predictors that captures the key features of these interactions and feedback loops. We investigate several common classes of predictors, including neural networks and nearest-neighbor models. Through experiments and theoretical analysis, we demonstrate that competition leads to specialization: while predictors perform better for specific sub-populations, they perform worse on the general population distribution compared to when there is no competition. Moreover, we show that the quality-of-service experienced by users in this ecosystem of ML predictors is non-monotonic with respect to the number of competing predictors. The quality-of-service for users is diminished when there are too few or too many competing predictors. There is an optimal number of competing predictors that provides the best quality-of-service for users. This optimal number depends on several factors. One critical factor is how well the users can individually identify the predictor that's best suited for them.

\vspace{-7pt}
\paragraph{Contributions} As ML systems become ever more widely used, often by competing companies, it is increasingly important to model and characterize the effects of competition on ML. This topic is under-explored in ML. We summarize our main contributions as follows: 

\begin{enumerate}

    \item We introduce a novel model for competing predictors, which enables both large-scale experiments and theoretical analysis. Our model is generally useful for exploring statistical, algorithmic, and economic phenomena concerning the feedback dynamics between populations of competing predictors and users.
    
    \item Through empirical and theoretical analysis, we show that user decisions create a feedback loop through which each ML predictor specializes toward a particular sub-population over time; often at the cost of worse performance over the general population of users.
    
    \item We analyze the effect of competition on the quality-of-service for the users. We show that the overall quality can be non-monotonic in the number of competing ML predictors.  

\end{enumerate}

\section{Model for competing predictors}\label{sec:model}

We assume there is some supervised ML task that requires algorithms to make predictions for users. The prediction task corresponds to a general population distribution $\mathcal{D}$. For $(x,y) \sim \mathcal{D}$ we can think of $x \in \mathcal{X}$ as representing the relevant user attributes or features and $y \in \mathcal{Y}$ as the predictive target. We have $k$ competing predictors, $\{A^{(1)}, \dots, A^{(k)}\}$. Predictor $i$ has an initial batch of training data $D^{(i)}_0$ that are independently and identically distributed (i.i.d.) samples from $\mathcal{D}$. The initial training data $D^{(i)}_0$ corresponds to the data that each predictor starts with---e.g. data from an initial pilot. We typically think of $|D^{(i)}_0|$ as small. We refer to this initial data as \emph{seed data}. Let $D^{(i)}_t$ denote the dataset that the $i$-th predictor has up to and including time $t$.  $A^{(i)}_t$ is the predictor that is trained on $D^{(i)}_{t-1}$.  At each time $t$, a new sample $(x_t,y_t) \sim \mathcal{D}$ is drawn, representing the $t$-th user in some user stream. Each predictor outputs $\hat{y}^{(i)}_t = A^{(i)}_t(x_t)$. Then, the user selects one of the $k$ predictors as a \emph{winner}, denoted by $w_t$.  The winning predictor $w_t$ gets the datum $(x_t, y_t)$: $D^{(w_t)}_t = D^{(w_t)}_{t-1} \cup \{(x_t, y_t)\}$ and $D^{(i)}_t = D^{(i)}_{t-1}$ for $i \neq w_t$. We can think of predictors as agents seeking to maximize their query rate and users as agents seeking to maximize the accuracy of the predictor they select. We model predictors with both common parametric and non-parametric ML algorithms. For simplicity, in our experiments and theory we will consider competitions in which predictors are \emph{symmetric}, meaning they use the same learning algorithm. We proceed to describe our user model.



\vspace{-7pt}
\paragraph{A flexible model for user choice}
We would like to model how a user chooses among the set of competing predictors. For starters, assume that $\mathcal{Y}$ in the prediction task is categorical and the \emph{prediction quality}, $q(y_t, \hat{y}_t) = \mathbf{1}\{y_t = \hat{y}_t\}$, is binary. We consider the case when users do not have prior biases towards any predictor. Instead, we stipulate that the probability that a user selects  predictor should only depend on the tuple $\mathbf{q}_t = (q(y_t, \hat{y}^{(1)}_t),...,q(y_t, \hat{y}^{(k)}_t))$, meaning that user selection probability is only a function of the prediction quality. We denote the user selection operation $\mathbf{SELECT}$. The $\mathbf{SELECT}$ encodes the conditional distribution for $w_t$ over $[k]$ given $\mathbf{q}_t$ where $[k] := \{1, \dots, k\}$. We can think of $\mathbf{SELECT}$ as a randomized operation that outputs the winner, \textit{i.e.}, $w_t = \mathbf{SELECT}(\mathbf{q}_t)$. Equivalently, $w_t$ is a random variable parameterized by $\mathbf{q}_t$. Given that $w_t$ only depends on $\mathbf{q}_t$, the sole parameter that uniquely characterizes a user's choices is the difference in probability that the user selects a correct predictor over an incorrect predictor. We refer to this as the \emph{correctness advantage}, $\mathbf{P}_{\text{ADV}}$, in the system. For any $\mathbf{q}_t$ such that for some $i \neq j$, $y_t = \hat{y}_t ^{(i)}$ and $y_t \neq \hat{y}_t ^{(j)}$, we define \emph{correctness advantage}\footnote{If either $\hat{y}_t ^{(i)}=y_t$ or $\hat{y}_t ^{(i)}\neq y_t$ for all $i \in [k]$, then a user chooses a competitor uniformly at random.} as

\begin{align*}
    \mathbf{P}_{\text{ADV}} := \mathbf{Pr}(w_t = i | \mathbf{q}_t) / \mathbf{Pr}(w_t = j \mid \mathbf{q}_t).
\end{align*}

Without loss of generality, we can equivalently use the widely-used softmax parameterization for $\mathbf{P}_{\text{ADV}}$:

\begin{align*}
    \mathbf{Pr} \left( w_t=i \mid \mathbf{q}_t \right) = \frac{1}{Z} e^{\left(\alpha q(y_t,\hat{y}^{(i)}_t) \right)}
\end{align*}

where $Z = \sum_{j \in [k]} \exp{ \left(\alpha  q(y_t,\hat{y}^{(j)}_t) \right)}$ and thus $\mathbf{P}_{\text{ADV}} = \exp(\alpha)$. 

For simplicity, we use temperature parameter $\alpha$  in lieu of $\mathbf{P}_{\text{ADV}}$ throughout this work; this parametrization does not limit user behavior. To be clear, the user does not necessarily know the true $y_t$ (otherwise there may not be a need for the predictors). Moreover it is not necessary that the user observes all of the predictions $\hat{y}^{(i)}_t$ when making a selection. It is sufficient that the user has some side information on which predictors are likely to be correct. The degree of this correlation can be captured by $\alpha$. This model is simple and flexible, and it captures the essence of the interaction between predictors and users. We can view the temperature parameter $\alpha$ as indicating how informed the user selections are. When $\alpha = 0$, the user has zero information and uniformly at random selects a predictor. As $\alpha$ increases, the user is more likely to select the algorithm that makes the correct prediction. Therefore, $\alpha$ is a natural metric of \emph{information efficiency}  . In many settings, users might be more likely to select a predictor that makes a correct prediction than an incorrect predictor (\textit{i.e.} $\alpha \geq 0$).  
This might be because users have some private signals or experiences, and also because users typically want to pick the highest quality prediction. Because this is more realistic, we primarily focus on $\alpha \geq 0$ for our experiments and analysis. 


For simplicity, we will largely deal with temperature $\alpha$ for the remainder of the paper, while remembering the direct connection between $\alpha$ and the correctness advantage. Another advantage of the softmax parameterization is that it easily generalizes to regression settings by replacing $q$ with any generic loss function $\ell$ (such as MSE). In the main text of this work, we will assume $\alpha$ is a system constant and thus is fixed for all users. In Appendix B, we further generalize and let $\alpha$ depend on the particular user that is making the selection by sampling each user's $\alpha$ parameter from a standard normal distribution. This reflects that individual users have varying amounts of prior information about the predictors. We found that this yields in highly similar results (refer to Appendix B).

One simplification that we make in our model is that only the selected predictor receives $(x_t, y_t)$. There are several possible modeling variation on this: for example, one could allow the non-selected predictors to add $x_t$ (not $y_t$) to its database and this could be used for semi-supervised learning. One could also allow the user selection to depend not just on the current predictions but also on predictor reputation. Additionally, one could assume that only some fraction of users actually leave feedback, which would mean that the winner observes $y_t$ only some fraction of the time. These are interesting directions for follow up exploration. In this paper, we make the simplifications in order to capture the key essence due to competition in purely supervised learning.

\vspace{-3pt}
\section{Experiments}
\vspace{-3pt}

We present simulations of competing learners in the supervised (Sec. 3.1) and collaborative filtering settings (Sec. 3.2). We investigate the effects of competition on the predictors and the users and empirically characterize predictor specialization and non-monotonicity of the quality-of-prediction.

\vspace{-2pt}
\subsection{Supervised Learning}
\vspace{-2pt}

We use several popular benchmark datasets for $\mathcal{D}$: \texttt{Postures} \citep{gardner20143d, Dua:2019}, \texttt{Adult Income} \citep{Dua:2019}, and \texttt{FashionMNIST} \citep{xiao2017}. For \texttt{Postures} and \texttt{Adult Income} in particular, each datum corresponds to data from one individual, which is particularly appropriate for our motivating competition setting.  
We explore the effects of different information efficiency value $\alpha$. For each dataset, we fix a small number of i.i.d. seed samples (order $10^0$ - $10^2$) and run the simulation for a large number of rounds (order $10^3$ to $10^4$).  
We perform our experiments with the widely-used multi-layer perceptron (MLP) as an example of parametric predictors and nearest-neighbors (NN) as an example of  non-parametric predictors. 
In Appendix B we also report similar simulations conducted with a logistic regression model as well as full details of all the experiments. While there are many other classes of predictors to explore, we believe that the standard models used here cleanly capture the key insights.

\vspace{-5pt}
\paragraph{Competition drives predictor specialization}


\begin{figure*}[t]
 \begin{center}
  \begin{subfigure}[l]{0.99\textwidth}
    \includegraphics[width=\textwidth]{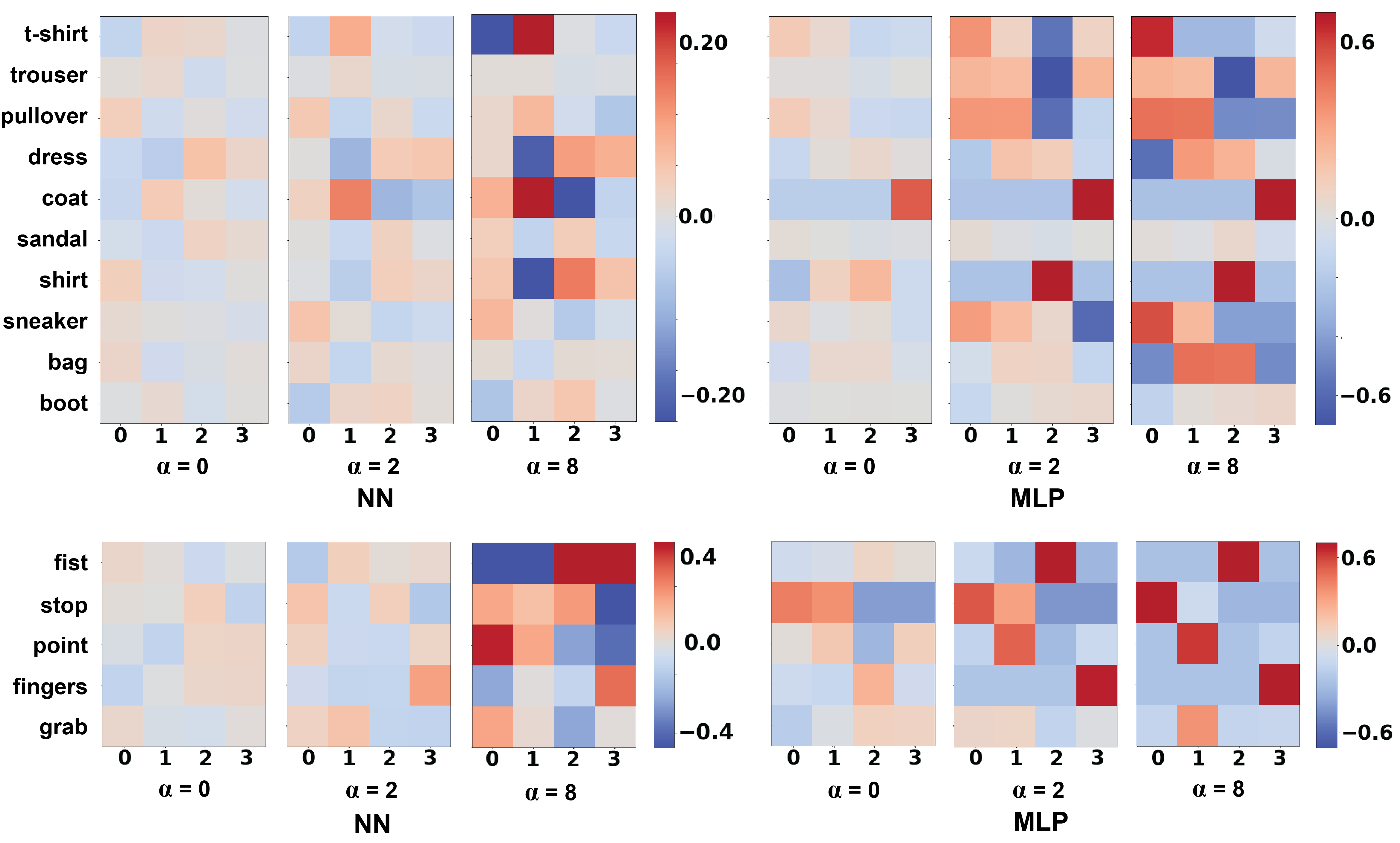}
  \end{subfigure}
  
\vspace{-3pt}  
\caption{  Predictor specialization heatmaps for FashionMNIST (top row) and Postures (bottom row) with NN (left column) and MLP (right column). For each dataset and algorithm we include heatmaps of $\alpha$ at low (0), medium (2), and high (8) values (left to right). Each heatmap is a \#(classes) $\times$  \#(predictors) grid. The $ij$-th block in a grid indicates the difference between the average class-conditional accuracy for the $i$-th class and the $j$-th predictor's class-conditional accuracy for the $i$-th class. \ Predictors are indexed by an arbitrary id number and classes are labeled on the left. Red (blue) indicates an accuracy that is higher (lower) than average, and white is average.
}
\label{fig:firm_spec_supervised}
 \end{center}
 \vspace{-12pt}
\end{figure*}

\begin{figure}[b!]
 \begin{center}
    \includegraphics[width=\columnwidth]{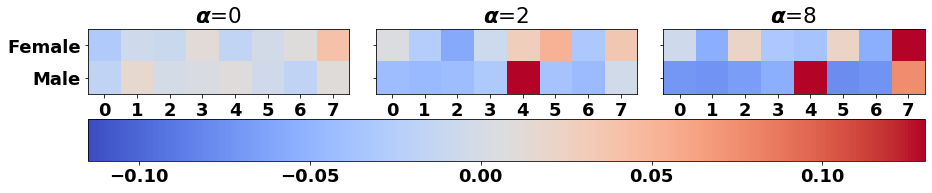}
\caption{ Predictor specialization heatmap for Adult with MLP with 8 competitors. Rows in the grid indicate male vs. female individuals. Red (blue) indicates an accuracy that is higher (lower) than average, and white is average.}
\vspace{-9pt}
\label{fig:gender_spec}
 \end{center}
\end{figure}

We performed experiments with four competing predictors (similar results are seen for other number of predictors). In Fig.~\ref{fig:firm_spec_supervised} we present heatmaps indicating the accuracy of the four competing predictors on each of the label classes. Red (blue) indicates that a predictor is better (worse) than the average predictor on that class.

\begin{figure*}[t]
 \begin{center}
  \begin{subfigure}[l]{0.99\textwidth}
    \includegraphics[width=\textwidth]{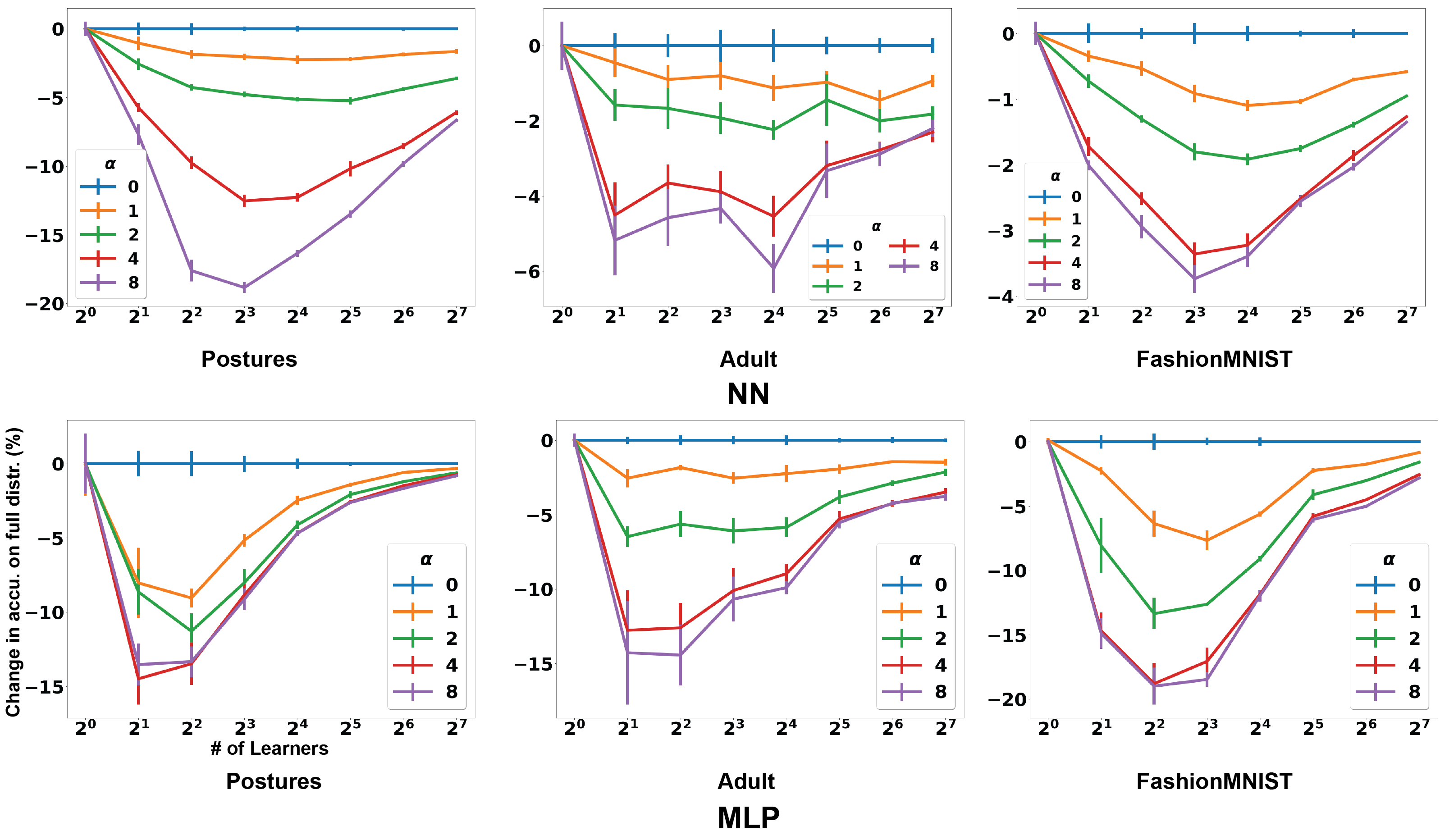}
  \end{subfigure}
 \vspace{-5pt}
\caption{  How specialization affects predictor performance: number of predictors (x-axes, log-scale) vs. change in accuracy over general population $\mathcal{D}$ (y-axes, in percentage) for NN and MLP on 3 datasets. To measure the effect of competition, change in accuracy is with respect to a baseline simulation in which winning predictors get an i.i.d. sample instead of the one that selected it to remove selection bias. Confidence intervals are standard error of the mean for 5 replicates.}
\vspace{-9pt}
\label{fig:cost_of_comp}
 \end{center}
\end{figure*}

When $\alpha = 0$, the user uniformly at random selects a predictor and the lack of competition results in all of the predictors being close to average accuracy.  As $\alpha$ increases, we see a clear trend towards greater variations in class-conditional accuracy among the predictors, indicating  specialization. A stark example of this can be observed for the competing MLPs on the Postures dataset. For large $\alpha$, the four predictors specialize over the five classes such that each predictor strongly favors only one particular class (except for predictor 3 which favors two classes). Predictor 0 specializes in detecting \texttt{stop}, predictor 2 specializes in \texttt{fist}, predictor 3 specializes in  \texttt{fingers} and predictor 1 is split between \texttt{point} and \texttt{grab}. Outside of each predictor's specialty class, the performance is low across the board. The predictor specialization not only occurs over the classes but also within the features. An interpretable example of this is for the binary gender feature in the \texttt{Adult Income} dataset (Fig.~\ref{fig:gender_spec}). At $\alpha = 0$ all predictors are close to average accuracy. As $\alpha$ increases, we see that predictor 4 and eventually predictor 7 specialize in males versus females, respectively. This illustrates how competition could lead to ML algorithms that specialized to specific demographic groups.

Larger $\alpha$ creates a positive feedback loop that leads to specialization. Random variation in the initial training batches generates some heterogeneity in the predictors. Users are likely to select the predictor that is best suited for them with large $\alpha$. This leads that predictor to improve its model specifically for that sub-population.  In turn, this results in an increased likelihood that members of that sub-population select said predictor. 
While a common business strategy is for firms to intentionally specialize to particular sub-populations from the onset  \citep{balassa1989comparative, yang2015specialization}, the interesting aspect of the phenomena here is that specialization emerges naturally (and unintentionally) due to the competition over data. 



We next quantify how the competition affects the predictor's performance on the general population distribution, which is measured as its average accuracy over $\mathcal{D}$. Note that this $\mathcal{D}$ is different from the distribution of data points from a user at any particular time --- as we shall see next, the predictor does better on the latter distribution.
Fig.~\ref{fig:cost_of_comp} measures the change in accuracy over $\mathcal{D}$ compared to the $\alpha = 0$ baseline, which uses the same number of training samples but removes competition.

There is a consistent trend that increasing the information efficiency $\alpha$ at any number of predictors results in lower accuracy on $\mathcal{D}$. The drop in accuracy is largest when there is an intermediate number of predictors. This is because the average number of samples each predictor receives decreases when there are more predictors, since the total number of rounds, or equivalently the total number of samples, is fixed. With fewer data points, there's less feedback to bias the predictor.
The decrease in accuracy for the overall distribution could be costly when the company tries to broaden its user-base to the entire $\mathcal{D}$. This is an important consequence of specialization.


\vspace{-9pt}
\paragraph{Prediction quality for users}
We shift our focus to analyze the prediction quality experienced by the users. We define the \emph{prediction quality for users} as the average accuracy of the selected predictor averaged over all the rounds of competition: $\frac{1}{T} \sum_{t=1} ^T \mathbf{1}( \hat{y}_t ^{(w_t)} = y_t)$. Fig.~\ref{fig:user_non_monotone} shows how this quality varies as the number of predictors  (x-axes) and $\alpha$ (different colors) change for NN and MLP applied to three datasets. In each panel, the total number of datapoints (i.e. users) is fixed. 
 Prediction quality for users is consistently higher when users have more information (larger $\alpha$) when picking the predictor.

 \begin{figure*}[t]
 \begin{center}
    \begin{subfigure}[l]{0.99\textwidth}
    \includegraphics[width=\textwidth]{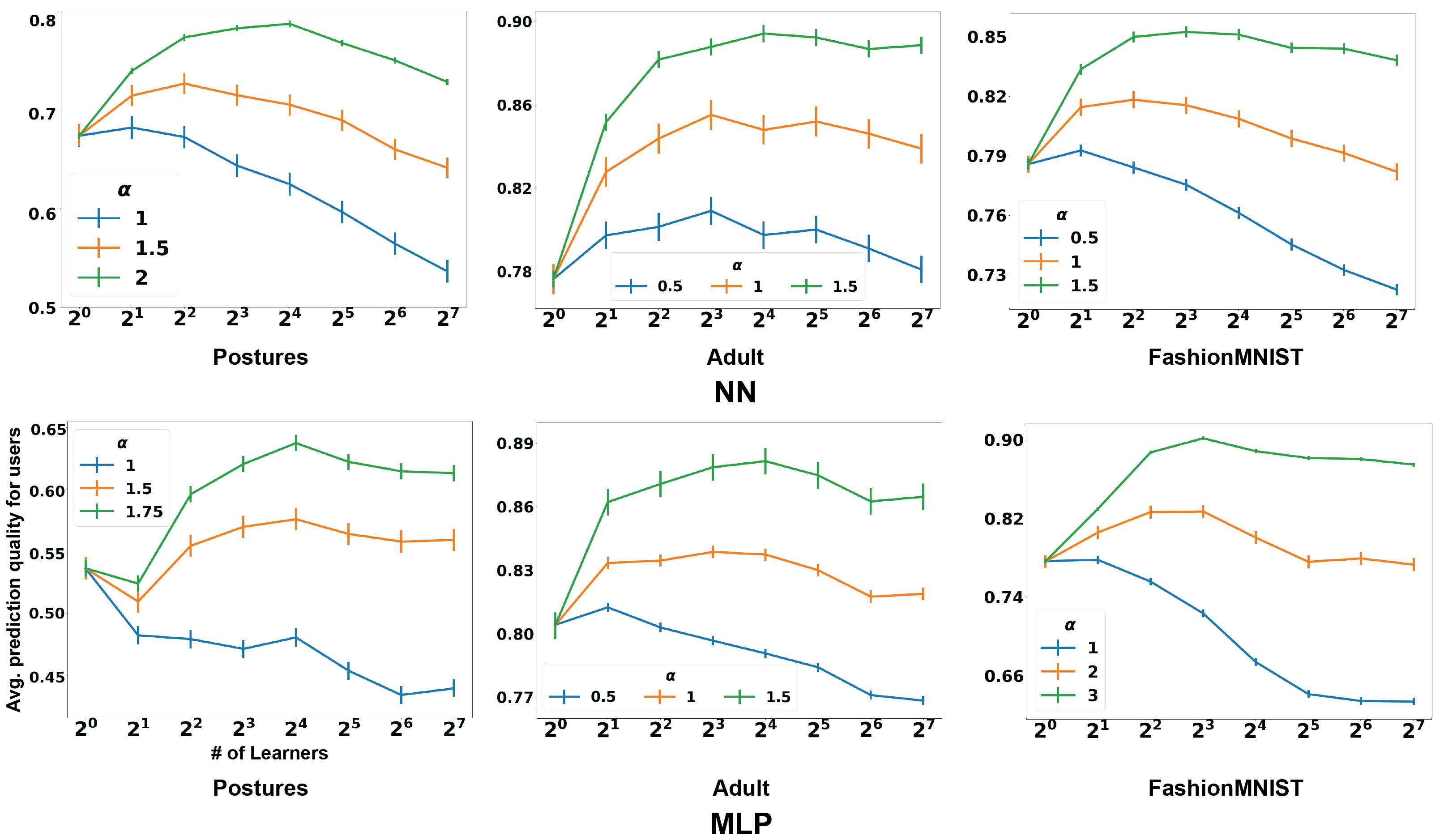}
  \end{subfigure}

\caption{  Prediction quality for users: number of predictors (log-scale) vs. avg. prediction quality for users with NN and MLP on 3 datasets. Prediction quality is averaged over all of the rounds in the simulation. Confidence intervals are standard error of the mean for 5 replicates.} 
  \label{fig:user_non_monotone}
 \end{center}
 \vspace{-15pt}
\end{figure*}

Interestingly, we find that the prediction quality for users can be non-monotonic. For example, in Postures data with competing NNs, the highest quality is achieved with 16 competing predictors; having too few or too many predictors decreases quality. The intuition for this phenomenon is as follows. When there is just one predictor, a user has no choice and changing $\alpha$ has no effect. With more predictors and relatively high information efficiency, each user can select the predictor that is likely to be accurate for it, and hence the prediction quality improves. However, when there are too many predictors, each predictor gets fewer training data (recall that the total number of data points is fixed). Hence none of the predictors is very accurate and the overall quality starts to decline. 
In Sec.~4, we show this phenomena is a  mathematical consequence of the learning competition under some mild conditions. The prediction quality over a full range of information efficiencies depicting the monotone increasing and decreasing regimes (for near-infinite and near-zero $\alpha$) can be found in Appendix B.




\vspace{-3pt}
\subsection{Extension to Collaborative Filtering}
\vspace{-3pt}




Previous experiments capture the setting where each user is a single data point that appears once. Here we experimentally investigate a collaborative filtering extension where each user contributes multiple data points. Collaborative filtering competitions follow the same structure as described in Sec. 2, with the primary differences being a new $\mathbf{SELECT}$ operation and allowing repeated samples from each user. As before, we have a set of $k$ competing recommenders. We also have a set of $m$ distinct users, $\{u^{(1)}, \dots, u^{(m)} \}$ that are seeking recommendations over a set of $r$ items (for simplicity, we assume that these items are shared across the recommenders). At each round, a uniformly at random user $u_t \in \{u^{(1)}, \dots, u^{(m)} \}$ selects one of $k$ recommenders: $w_t = \mathbf{SELECT}(u_t)$.
Then recommender $w_t$ recommends an item for $u_t$: $A_{t}^{(w_t)}(u_t) \in [r]$. There is a latent preference matrix $M \in [0,1]^{r \times m}$, where $M_{ij}$ is the probability that user $u^{(j)}$ interacts with the $i$-th item (pCTR). The ``winning" recommender, $w_t$ observes the interaction between a user and item as feedback. Precisely, recommender $w_t$ observes $(x_t,\Tilde{y}_t)$ where $x_t:=(i,j)$ is simply a pair of the item $i$ and the user $j$, and $\Tilde{y}_t \sim \mathbf{Bernoulli}(M_{ij})$ describes if there is an interaction when item $i$ is recommended to user $j$. As before: $D^{(w_t)}_t = D^{(w_t)}_{t-1} \cup \{ (x_t,\Tilde{y}_t) \}$ and $D^{(i)}_t = D^{(i)}_{t-1}$ for $i \neq w_t$.

Users want to maximize the preference scores of the items that they get recommended to them, and recommenders want to maximize the number of queries for items they receive from users. In our experiments, each user keeps track of the quality of past recommendations from each recommender and individually solves a multi-arm bandit \citep{MAL-068} problem with recommenders as arms when it is their turn to $\mathbf{SELECT}$. 
Each recommender similarly solves an  an online matrix factorization problem \citep{schafer2007collaborative} based on the observed user-item interactions using alternating least-squares \citep{hastie2015matrix}. 
We generate $M$ as the product of low-rank factors with i.i.d. Gaussian entries. 
We run the simulations for $2 \times 10^5$ rounds. Appendix B contains the formal description of the model and details about the protocol and implementation.

Fig.~\ref{fig:cf_case_study} shows the collaborative filtering results. Fig.~\ref{fig:cf_case_study} (left) is analogous to Fig.~\ref{fig:cost_of_comp}; the y-axes quantifies how well each recommender performs over the general population distribution of users. This performance is measured as the expected probability that a randomly selected user decides to interact with the item suggested by this recommender. As in the setting of competing predictors, competition and specialization leads to a decrease in the performance of recommenders for the general user distribution. Fig.~\ref{fig:cf_case_study} (right) is analogous to Fig.~\ref{fig:user_non_monotone}; the y-axes there is the prediction quality experienced by the users.
We find a similar phenomenon as before:  having too few or too many recommenders can decrease the quality experienced by users. These collaborative filtering experiments demonstrate that the phenomena that competition leads to algorithmic specialization and that there is a sweet spot for the number of ML models can hold in diverse settings.     


\begin{figure*}[h!]
 \begin{center}
 \vspace{-8pt}
     \begin{subfigure}[l]{0.95\textwidth}
    \includegraphics[width=\textwidth]{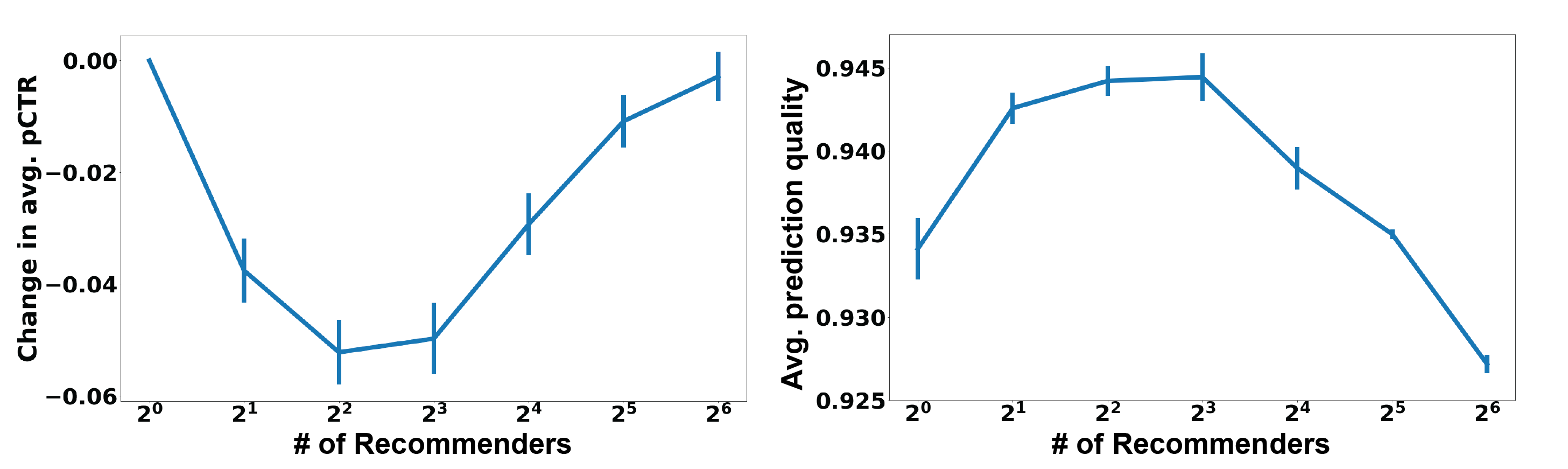}
  \end{subfigure}

\caption{  Collaborative filtering competition: Recommender pCTR over general population (left) and avg. prediction quality for users (right) for varying number of recommenders (log-scale). Change in pCTR (left) is with respect to an otherwise identical baseline simulation in which winning recommenders always get an i.i.d. user sample instead of the user that selected it. Prediction quality (right) is averaged out over all the round in the simulation. Confidence intervals are std. error of the mean from 5 replicates.} 

\label{fig:cf_case_study}
 \end{center}
 \vspace{-12pt}
\end{figure*}

\vspace{-3pt}
\section{Theoretical analysis}
\vspace{-3pt}

We carry out theoretical analysis to further understand and support our empirical findings. Here, we  assume a binary classification task for simplicity. Complete proofs for all claims are in Appendix C. The analysis in this section can be interpreted as formalizing sufficient conditions for the empirically observed effects of competition to emerge. 







\vspace{-4pt}
\subsection{Cost of competition for predictors}
\vspace{-4pt}

Our experiments show that competition causes each predictor to specialize on a sub-population and perform worse on the overall population distribution. We show for simple parametric and non-parametric models that competition results in a gap in the error rates attained by the trained predictors.
Let  $ \mathcal{R}(A; \mathcal{D}) = \mathbf{E}[{\mathbf{1}\{ A(X) \neq Y\} }]$, where $(X,Y) \sim \mathcal{D}$, denote the error rate of a predictor $A$ on samples from the general population. The average error rate of the competing predictors (on $\mathcal{D}$) after $t$ rounds of competition is $\mathcal{R}^{k}_t = \sum_{i \in [k]} \mathcal{R}(A^{(i)}_{t}; \mathcal{D})/k$, where $A^{(i)}_t$ is  predictor $i$ after $t$ rounds of competition as described in Sec.~\ref{sec:model} and $k$ is the total number of competitors. The following asymptotic result concerns itself with the perfect information limit $\alpha = \infty)$ and holds quite generally for most non-parametric models. Plainly speaking, the theorem says that for certain distributions, the average error rate of competing predictors is not within a constant factor of the error rate of a single predictor.


\begin{theorem} 
\label{thm:non_param_inf_gap}
Suppose users have perfect information  ($\alpha = \infty$) and each predictor is trained using a non-parametric method that is asymptotically a $C$-approximation (in the usual sense, see \citet{ausiello2012complexity}) to the Bayes error rate. 
Then, for any seed set size $s = |D_0|$, there exists $\mathcal{D}$ such that for any $k > 1$, and , $\lim_{t \rightarrow \infty} \frac{\mathcal{R}^{k}_t}{\mathcal{R}^1_t} = \infty$.
\end{theorem}


\vspace{-2pt}
The intuition for Thm.~\ref{thm:non_param_inf_gap} is as follows. In the case that $Y$ is deterministic given $X$, the ML problem is effectively an interpolation. In this case, the Bayes error rate is $0$ and this error rate is asymptotically achieved by most non-parametric methods \citep{tsybakov2008introduction} given that they are $C$-approximations to the Bayes rate. However, when $\alpha = \infty$ in a competition, an unlucky seed set could result in a predictor never achieving $0$ error rate, which breaks the $C$-approximation. Furthermore, this probability can be bounded away from $0$ for any finite seed set.  Next we show that a risk gap still exists for finite $\alpha$.
\begin{theorem} 
\label{thm:non_param_inf_gap2}
Suppose $k=2$ and both predictors use the nearest-neighbor algorithm. Let $s = |D_0|$ be the number of i.i.d. seed samples that each predictor starts with and assume $ s \geq 2$. If $\alpha > \log (2)$,  then there exists $\mathcal{D}$ such that $$
\lim_{t \rightarrow \infty} \frac{\mathcal{R}_{t}^2}{\mathcal{R}^{1}_t}  \geq 1 +  \frac{1}{54\sqrt{2s}}\left(\frac{8}{9\sqrt{s}}  \right)^{s/2}\left(1 - \frac{2}{2 + e^\alpha}\right)^2
$$


\end{theorem}


\vspace{-2pt}
The risk ratio decreases quickly in $s$, indicating that sufficient seed data may be an effective counter-measure for the non-parametric case. Also, the risk ratio grows larger for larger $\alpha$, which also coincides with intuition. 

We next investigate the parametric setting by analyzing an ordinary linear least squares regression. For this analysis, we use the mean squared error to measure expected risk $\mathcal{R}^k_t$. We present two lower bounds. The first holds for any positive information efficiency $\alpha > 0$ and depends on the number of seed samples. The second holds for any finite number of seed samples, but requires the users to have perfect information ($\alpha = \infty)$.

\begin{theorem} 
\label{thm:OLS_gap_perf_info}
Suppose the data is generated from a  linear model $Y = XW + \epsilon$ with $\mathbf{E}(\epsilon|X) = 0$. 
Assume each predictor uses an ordinary least-squares linear estimator. Let $ s \geq 1$ be the number of i.i.d. seed samples each predictor starts with and assume  $k \geq 2$. We have the following:

(i) If $\alpha > 0$ then  $\lim_{t\rightarrow \infty} \sup_{\mathcal{D}} \frac{\mathcal{R}^{k}_t}{\mathcal{R}^1_t}  \geq 1 + \frac{1}{7056s^{3/2}}  $ 

(ii) If $\alpha = \infty$ then $\lim_{t \rightarrow \infty} \sup_{\mathcal{D}} \frac{\mathcal{R}^{k}_t}{\mathcal{R}^1_t}  \geq  \frac{2k}{k+1}$
\end{theorem}

Thm.~\ref{thm:OLS_gap_perf_info} tells us that when $\alpha$ is large, there is a significant (close to $2\times$) penalty incurred when there are many competing predictors. When $\alpha$ is small, the penalty is also smaller but does not vanish if the number of seed samples is not too large. Notice that for regression, the worst-case ratio of expected risks vanishes at a low-degree polynomial rate in $s$. This decays far slower than the exponentially vanishing bound for non-parametric methods.
This suggests that seed data may be less helpful in mitigating the cost of competition with parametric methods than with non-parametric.




\vspace{-10pt}



\subsection{Prediction quality for users with competing predictors}
\vspace{-3pt}

We analyze how the number of competing predictors affects the overall prediction quality experienced by users. We want to characterize the dependence of quality  on the number of predictors, $k$, and the information efficiency, $\alpha$. Recall our notion of empirically measurable \emph{prediction quality for users}: $\frac{1}{T} \sum_{t=1} ^T \mathbf{1}( \hat{y}_t ^{(w_t)} = y_t)$. Here we will be studying theoretically relevant quantities to this random empirical value. We define the \emph{expected prediction quality at time $\tau$}, denoted by $\mathbb{A}_\tau$ as $ \mathbb{A}_\tau = \mathbf{E}(\mathbf{1}\{ \hat{y}_\tau ^{(w_\tau)} = y_\tau\})$. To this end, we will phrase our results in terms of the accuracy, $\mathcal{A}^{k}_t$, defined by $\mathcal{A}^{k}_t = 1 - \mathcal{R}^{k}_t$ rather than the risk $\mathcal{R}^{k}_t$.  


\paragraph{Assumptions} To make the analysis tractable, we make several natural modeling assumptions that we outline here. We define the following:  $\delta = \mathcal{A}_t^1 - \mathcal{A}_t^2$ and $\varepsilon = \mathcal{A}^k_0 - \frac{1}{2} $.

\begin{enumerate}
    \item  We are primarily interested in regimes when seed sets are small, which implies that the initial predictors are weak models. Concretely, we assume: $\varepsilon < 1/14$
    \item Also, we should have enough data to experience diminishing marginal returns from additional samples. This means that the individual accuracy for one predictor is not much better than the individual accuracy for two predictors each with approximately half as many samples. Concretely, we assume: $0 < \delta < \frac{1}{6}$
    \item While we allow the predictors to be correlated, they cannot be extremely correlated. To see why this is necessary, consider the case in which the predictors are perfectly correlated. They always give the same prediction and thus the users derive no benefit from the competition.
    \item Finally, we assume that the expected accuracy for a predictor monotonically increases in the data set size. Thus, having more data is better, \emph{on average}, but not necessarily always. 
\end{enumerate}


Our result shows that in the regimes described above, there necessarily exists an interval of intermediate information efficiencies,  $0 < c_1 < c_2 < \infty$ such that for $\alpha \in (c_1,c_2)$, the optimal number of predictors is neither $1$ or $\infty$. This means that that there is a finite ``sweet spot'' in the number of competing predictors that produce the best user quality. 
\begin{theorem}

\label{thm:user_welfare}

Assume a learning competition at round $t$ under the conditions stated above. Let $\rho$ be the pairwise covariance between two predictors. If we have $\rho < \mathcal{A}^k_t - (\mathcal{A}^k_t)^2 - 6 \delta$ then there exists $0 < c_1 < c_2 < \infty$ such that if $c_1 < \alpha < c_2$ then the expected prediction quality for users at round $t$ is maximized by some $k^*$ number of predictors such that $1 < k^* < \infty$. In particular, $c_1 < \log \frac{\mathcal{A}_t^1 -(\mathcal{A}_t^1 - \delta)^2 - \rho}{\mathcal{A}_t^1 -(\mathcal{A}_t^1 - \delta)^2 - \rho - 2\delta}$ and $c_2 >  \log \frac{(1-4\varepsilon)\mathcal{A}_t^1}{1- \mathcal{A}_t^1}$.
\end{theorem}{}

To make the result concrete, we instantiate $\mathcal{A}^{1}_t \gets 0.9$, $\delta \gets 0.05$, $\epsilon \gets 0.05$ and $\rho \gets 0$. Thm.~4.4 tells us that prediction quality for users at time $t$ is non-monotonic  if $0.65 < \alpha < 1.97 $. This range of $\alpha$ agrees reasonably well with our empirical measurements. The intuition for the theorem is as follows. Obviously, when $\alpha$ is large having many weak predictors is better for users as the users themselves can take the burden of selecting a correct predictor. When $\alpha$ is not too large, having many weak predictors is not necessarily better for users than having a few smart ones (consider the extreme case of $\alpha \rightarrow 0$). However, if $\alpha$ is exactly zero, then having a single predictor is generally better than having even two predictors since the user is not more likely to $\mathbf{SELECT}$ the correct predictor and the two predictors have split the data. But, there is a sweet spot in $\alpha$ for which the user benefit from being slightly more likely to select the correct predictor outweighs the benefit that a single predictor has in terms of volume of training data. This is due to the near-universal phenomena in ML of diminishing marginal returns in number of training samples.
\vspace{-3pt}
\section{Discussion}
\vspace{-3pt}


\paragraph{Related works} In Mixture-of-experts and related ensemble learning methods, multiple predictors work together to train for a prediction task \citep{masoudnia2014mixture,dietterich2000ensemble,zhou2012ensemble,opitz1999popular}. There, the algorithms work together in the ensemble to optimize a common objective, and data can be shared between the algorithms. This differs from our setting where the predictors directly compete over user queries and training data. 

Recent literature in multi-agent reinforcement learning (MARL) has largely focused on emergent behaviour in collaborative dynamics between multiple agents \citep{collaborative, Nguyen_2020, dualopt, multioverview, bansal2017emergent, baker2019emergent, foerster}. In the fully-competitive setting, MARLs are typically modeled as zero-sum Markov games, and span a variety of applications such as exploration \citep{baker2019emergent, 8606991}, control \citep{Hrabia2018}, and others \citep{robust, market}. 
Existing RL approaches in multi-agent competition have studied competitions between two agents \citep{littman, bandits, aridor2019perils} with a focus on the expected equilibrium outcome and agent strategies. In particular, \citet{dong2019equilibrium} proposes that the Nash equilibrium for two firms in similarly motivated data acquisition learning game tends toward monopoly at the expense of consumer welfare. We differ from this line of work by explicitly modeling both the predictors and user decisions, incorporating user and sampling biases into our model, and by allowing for any number of predictors and users. This flexibility is critical as we find that the quality of prediction experienced by users heavily depends on the number of competing predictors. Another substantial difference between our analysis and that proposed in \citet{dong2019equilibrium} is that ours takes into the account the particular structure of a given supervised learning algorithm. On the other hand, the analysis in \citet{dong2019equilibrium} generically assumes learning algorithms can be replaced by black-boxes that simply behave according to canonical minimax error rates.




Another body of work focuses on examining and addressing single-agent direct feedback loops present in sample selection, namely sampling bias \citep{nie2018adaptively,shin2019sample,canonicalsamplebias, NIPS2014_5458, ecosamplebias, NIPS2006_3075, 10.5555/2976456.2976532, 10.1007/978-3-540-87987-9_8, 10.2307/146317}, but the problem remains under-explored in the case of multi-agent competition. Other forms of a feedback loop in ML systems that have been explored include social media filter bubbles \citep{techdebt}, risk assessment \citep{10.1145/3287560.3287563}, and algorithmic policing \citep{urns}. Dueling algorithms have been explored in \citet{immorlica2011dueling}, though they did not consider any statistical learning settings.


\paragraph{Extensions, limitations and future works} This paper proposes a model of competing predictors that enables both empirical and theoretical investigations. We characterize several interesting phenomena, namely how competition leads predictors to specialize and how too little or too much competition can both hurt the quality of prediction experienced by users. The phenomena that we capture, both empirically and theoretically, have not been studied in depth before and are interesting to the general ML community.

Because this is a relatively new direction of research in ML, we make several simplifications that allow the model to capture the essence of competition without overly complicating the insights. Most of our experiments and theory focus on the setting where each user corresponds to a single data point and only appears once. This is reasonable in applications with large populations of users and relatively infrequent repeated interaction. We conduct collaborative filtering experiments in which users and recommenders repeatedly interact over time, and find the phenomena remain. Additional investigation of repeated interactions is a fertile direction for future study.   

A simplification we have made is that predictors do not directly interact with other predictors except through their competition over data. In practice, companies behind ML predictors may merge, intentionally differentiate (which could lead to further specialization), or spend money to acquire data. A more general model that captures the full game dynamics would define strategy spaces and payoffs for each predictor and user, and characterize incentive compatible strategies.
Finally, we have assumed that the predictor that is selected receives the true label. In practice, there could be additional noise and time lag in the outcome that the predictor observes. This could also be interesting to model. 

As prediction algorithms become increasingly wide-spread, how they interact with each other and the  consequences of such competition are very important topics to explore. 

\vspace{-3pt}
\section*{\small Acknowledgements}
\vspace{-3pt}

\small{A.G. is supported by a Stanford Bio-X Fellowship. Y.K. is supported by NIH R01HG010140. J.Z. is supported by NSF CCF 1763191, NSF CAREER 1942926, NIH P30AG059307, NIH U01MH098953 and grants from the Silicon Valley Foundation and the Chan-Zuckerberg Initiative. We would like to thank M. Tiwari and B. He for helpful discussion in the early stages of this work. We would also like to thank A. Grosskopf for editorial help with the manuscript.}
\normalsize

\onecolumn
\appendix
\section{Extended Discussion}
Due to space constraints, we continue a supplementary discussion here (Appendix A). Experimental details and supplementary figures can be found in Appendix B and mathematical details can be found in Appendix C.

\subsection{Economic \& Multi-Agent Theory } Many of the concepts and quantities we study in this work have parallels in economic theory \cite{Suzumura1996, harberger1954, Lee2008,negishi1989monopolistic,salop1976information,chamberlain1946they, dixit1977monopolistic,wolinsky1986true,prat2019attention, yarrow1985welfare}. For example, quality of prediction is a notion of consumer welfare. The  economics literature tends to focus on competition between firms or mathematical agents, rather than on specific ML predictors. It would be an interesting direction for future work to connect and extend our learning competition framework from the economic perspective.

\subsection{Softmax Model} We include additional details and justifications for the model proposed in Section 2. Our proposed model of user choice satisfies Luce's choice axioms \cite{pleskac2015decision,luce1960individual,luce2012individual,morgan1974luce} and emerges from the established information-theoretic notion of rational inattention in economics \cite{sims2010rational,sims2006rational,sims2003implications,mackowiak2009optimal}. The softmax form in particular can be derived from optimal decision making under information processing constraints \cite{ortega2013thermodynamics,ortega2011unified, ortega2011information}. 

\subsection{Feedback Loops in ML Systems} Feedback loops, where two systems repeatedly influence and have access to only the decisions of the other, have been studied in supervised and reinforcement learning \cite{techdebt}. A related and interesting example includes the feedback in online reviews on digital media platforms \cite{luca2016fake, xie2016online, filieri2015travelers}. Existing works examine feedback loops in the single-agent setting \cite{degen}, with particular branches proposing metrics to recover counterfactuals of consumer preferences fixing consumer strategy \cite{counterfactual, deconv}. The effects of competition in ML holds significant implications  to sociology \cite{10.2307/2095230, BaffoeBonnie2008}, economics \cite{10.2307/135288, lee2001self, yoo2007estimation}, electoral systems \cite{PIANZOLA2014272, kellner2004can}, and recommendation systems \cite{kramer2007self, ruiz2014understanding, steck2011item, steck2013evaluation}. \cite{matchmarkets} studies an interesting but substantially different model of bandits in matching markets. The model we study here is also distinct from standard settings for online learning or active learning, where typically a single algorithm gets to explore and select data. In our setup, each data selects one among several competing algorithms.

\section{Experimental Details}

The implementation for the experiments may be found on GitHub in repository: $\texttt{tginart/competing-ai}$.

\subsection{Supervised learning}
We ran three datasets for our supervised learning simulations:  \texttt{Postures} \cite{gardner20143d, Dua:2019}, \texttt{Adult Income} \cite{Dua:2019}, and \texttt{FashionMNIST} \cite{xiao2017}. \texttt{Postures} has 5 classes. \texttt{Adult Income} has 2 classes. \texttt{FashionMNIST} has 10 classes. In the main text we reported results for NN and MLP. Our implementation is in Python \cite{van2007python} using Pytorch \cite{paszke2019pytorch} and Numpy \cite{walt2011numpy} frameworks.

 \begin{algorithm}[h]
 \caption{Competing predictors in supervised learning }
 \begin{algorithmic}
 \footnotesize
  \STATE \textbf{Input:} Set of competing predictors $\{ A^{(1)},\dots,A^{(k)} \}$, general population distribution $\mathcal{D}$, prediction quality function $q$.
  \STATE \textit{\# To initialize, all predictors warm-start with $s \in \mathbb{N}$ i.i.d. seed data}
  \FOR{$i \in [k]$}
     \STATE IID random sample $\mathcal{D}_{\mathrm{seed}} ^{(i)} := \{(x_{0,j} ^{(i)},y_{0,j} ^{(i)})\}_{j=1} ^s \sim \mathcal{D}$
     \STATE Train a predictor $A_1 ^{(i)}$ using $\mathcal{D}_{\mathrm{seed}} ^{(i)}$.
     \ENDFOR
 
 \STATE \textit{\# The competition begins}
     \FOR{$1 \leq t < T $}
     \STATE Random sample $(x_t,y_t) \sim \mathcal{D}$ \textit{\# User samples from the general population}
     \STATE For all $i \in [k]$, predictor $A_t ^{(i)}$ proposes a prediction $\hat{y}_t ^{(i)}$, \textit{i.e.}, $\hat{y}_t ^{(i)} = A_t ^{(i)}(x_t)$
     \STATE User selects a winner $w_t \in [k]$ based on $\{ q(y_t,\hat{y}_t ^{(1)}), \dots, q(y_t,\hat{y}_t ^{(k)})\}$  \textit{\# User decides which predictor to query based on the prediction quality}
     \STATE Winner updates the current model $A_{t} ^{(w_t)}$ and defines $A_{t+1} ^{(w_t)}$ by retraining with a new datum $(x_t, y_t)$
     \ENDFOR{}
 \end{algorithmic}
 \label{alg:CLM}
 \end{algorithm}

\subsubsection{Hyperparameters and Training Protocols}

\paragraph{Seed samples} For each data set we set a number of seed samples that was sufficient to train a model to perform slightly (a few percentage points) better than random guessing. For \texttt{Postures} and \texttt{Adult} we set a seed size of 3 samples. For \texttt{FashionMNIST} we set a seed size of 100. 

\paragraph{NN} For NN, we always use \emph{one} nearest-neighbor to keep things simple. Because the method is non-parametric, training the model just consists of appending each new sample to the data matrix.  Refer to \cite{friedman2001elements,altman1992introduction} for more details on the nearest-neighbor algorithm.

\paragraph{MLP} For the MLP, we use the same architecture hyper-parameters for each $\alpha$ in order to ensure consistency. All predictors used the same hyper-parameters and training protocol as well. We selected these by doing a small amount of manual tuning. We use Pytorch's Adam optimizer \cite{kingma2014adam} to train our MLPs (with default settings). For all datasets, we use 1 hidden layer. At the start of the competition, we used Pytorch's default initialization. After this, we never reinitialized the weights. Instead, we always fine-tune the weights from the previous training pass. Refer to \cite{goodfellow2016deep} for more details about MLPs.

For \texttt{Postures} the input width is 16 and the hidden width is 16. We used a learning rate of $10^{-3}$. After the initial training on the seed samples, we retrained after every 4th new data point was added to the set. When training, we used we batch size of (up to) 32. We randomly shuffled the data for each training instance. We trained until we had updated on $1,000$ data points or reached 32 epochs (whichever was met first -- this depended on the number of data samples the predictor had observed).  

For \texttt{Adult Income} the input width is 50 and the hidden width is 64. We used a learning rate of  $10^{-3}$. After the initial training on the seed samples, we retrained after every 32nd new data point was added to the set. When training, we used we batch size of (up to) 32. We randomly shuffled the data for each training instance. We trained until we had updated on $1,000$ data points or reached 32 epochs. 

For \texttt{FashionMNIST} the input width is 784 and the hidden width is 400.  We used a learning rate of  $10^{-4}$. After the initial training on the seed samples, we retrained after every 500th new data point was added to the set. When training, we used we batch size of (up to) 50. We randomly shuffled the data for each training instance. We trained until we had reached 30 epochs.

\paragraph{Number of Iterations} The simulation takes time (roughly) quadratic in the number of iterations because of the frequent retraining that takes place. For \texttt{Postures} and \texttt{Adult}, we run the competition for 2,000 rounds when using the NN and we run the competition for 4,000 rounds when using the MLP. For \texttt{FashionMNIST} we run the competition for 10,000 rounds for both algorithms.

\subsection{Collaborative Filtering}

We describe the protocol for the collaborative filtering experiments in detail. Recall that at each round $t$, one of $m$ users is sampled uniformly at random. The sampled user at time $t$, denoted $U_t$, must then query one of $k$ recommenders. The selected recommender then recommends one of $r$ items. We have an underlying preference matrix, $M \in [0,1]^{r \times m}$ such that $M_{ij}$ is the $j$-th user's preference score for item . The goal of the users is to maximize the the preference score 

\paragraph{Preference Matrix} For our simulation we set $r \gets 12$ and $m \gets 64$. We sweep over $k$ as seen in Fig. 4. We sample $W_{ij} \sim N(0,1)$ for $W \in \mathbb{R}^{m \times 3}$ and we sample $V_{ij} \sim N(0,1)$ for $V \in \mathbb{R}^{r \times 3}$. We compute $M' = VW^T$ and then affine scale $M'$ onto $[0,1]$ to produce the final preference matrix $M$. The final rank of $M$ is 4 (the rank increases by 1 due to the affine scaling). 

\paragraph{Users} Each user operates independently of the others without sharing data or otherwise communicating. Users treat each recommender as an arm in a $k$-arm bandit problem. Users operate as if the arms are stationary in time (although this is actually false since the recommenders can improve in time). The reward each user obtains from each recommendation is dictated the underlying preference matrix. Each user uses an $\epsilon$-greedy strategy to $\mathbf{SELECT}$ a recommender. When user $i$ is selecting $\epsilon \gets \tau_{i}^{-0.3}$ where $\tau_{i}$ is the total number of rounds in which user $i$ has participated thus far.

\paragraph{Recommenders} Recommenders use online matrix factorization to reconstruct the underlying preference matrix. Each recommender uses an $\epsilon$-greedy strategy to recommend items to users based on their current representation. Recommenders do not directly observe $M_{ij}$ when they recommend item $i$ to user $j$. Instead, they observe noisy feedback sampled from a Bernoulli trial with mean $M_{ij}$. Recommender $a$ has their own copy of $\hat W^{(a)} \in \mathbb{R}^{m \times 4}$ and $\hat V^{(a)} \in \mathbb{R}^{r \times 4}$ which are initialized with independent uniform random entries. See algorithm below for the psuedo-code for this implementation. The dependence on $a$ is implicit (again, each recommender solves their own instance of the problem with their own data).

 \begin{algorithm}[h]
 \caption{Online Matrix Factorization Update for Collaborative Filtering}
 \begin{algorithmic}
 \footnotesize
  \STATE \textbf{Input:} User index $i$, item index $j$, $\tilde Y \sim \mathbf{Bern}(M_{ij})$, number of observations for pair $(i,j)$, denoted $\nu_{ij}$, learning rate $\gamma$ and regularization penalty $\lambda$  
  
  \STATE

  \STATE $\hat M_{ij} \gets \frac{\nu_{ij}-1}{\nu_{ij}}\hat M_{ij} + \frac{1}{\nu_{ij}} \tilde Y $ \textit{\# Update running average pair $(i,j)$}
  \STATE $\hat V_i \gets \hat V_i - \gamma(\hat M_{ij} - V_iW^{T}_j)\hat W_j - \lambda \hat V_i $  \textit{\# Update row from user matrix}
   \STATE $\hat W_j \gets \hat W_j - \gamma(\hat M_{ij} - V_iW^{T}_j)\hat V^{(a)}_i - \lambda \hat W^{(a)}_j $  \textit{\# Update row from item matrix}

 \end{algorithmic}
 \label{alg:MF}
 \end{algorithm}

 For our simulation we use $\lambda = 10^{-4}$ and $\gamma = 0.1$. When recommender $a$ chooses an item to select for user $i$ they select $\max_{j}\{ \hat V^{(a)}_i  \hat W^{(a)T}_j \}$ with probability $1-\epsilon$ and a uniformly random item with probability $\epsilon$. For consistency with the users, we set $\epsilon \gets \tau_{a}^{-0.3}$ where $\tau_{a}$ is the number of recommendations given out by the recommender $a$ thus far.

\subsection{Additional Experiments}
We include some additional experiments and figures. In particular, we sweep wider ranges of $\alpha$ and we also include some additional simulations using logistic regression as the prediction algorithm.

In Fig.~\ref{fig:app_log_reg_experiment} we report the experiment in Fig. 3 using the logistic regression model. The trends are similar to what we saw with the MLP and NN models.
In Fig.~\ref{fig:app_larger_sweep} we report the experiment in Fig. 4 but for a wider range of $\alpha$. When $\alpha$ is large the prediction quality for users is monotone increasing and when $\alpha$ it is monotone decreasing.

\begin{figure}[h]
 \begin{center}
  \begin{subfigure}[l]{0.45\textwidth}
    \includegraphics[width=\textwidth]{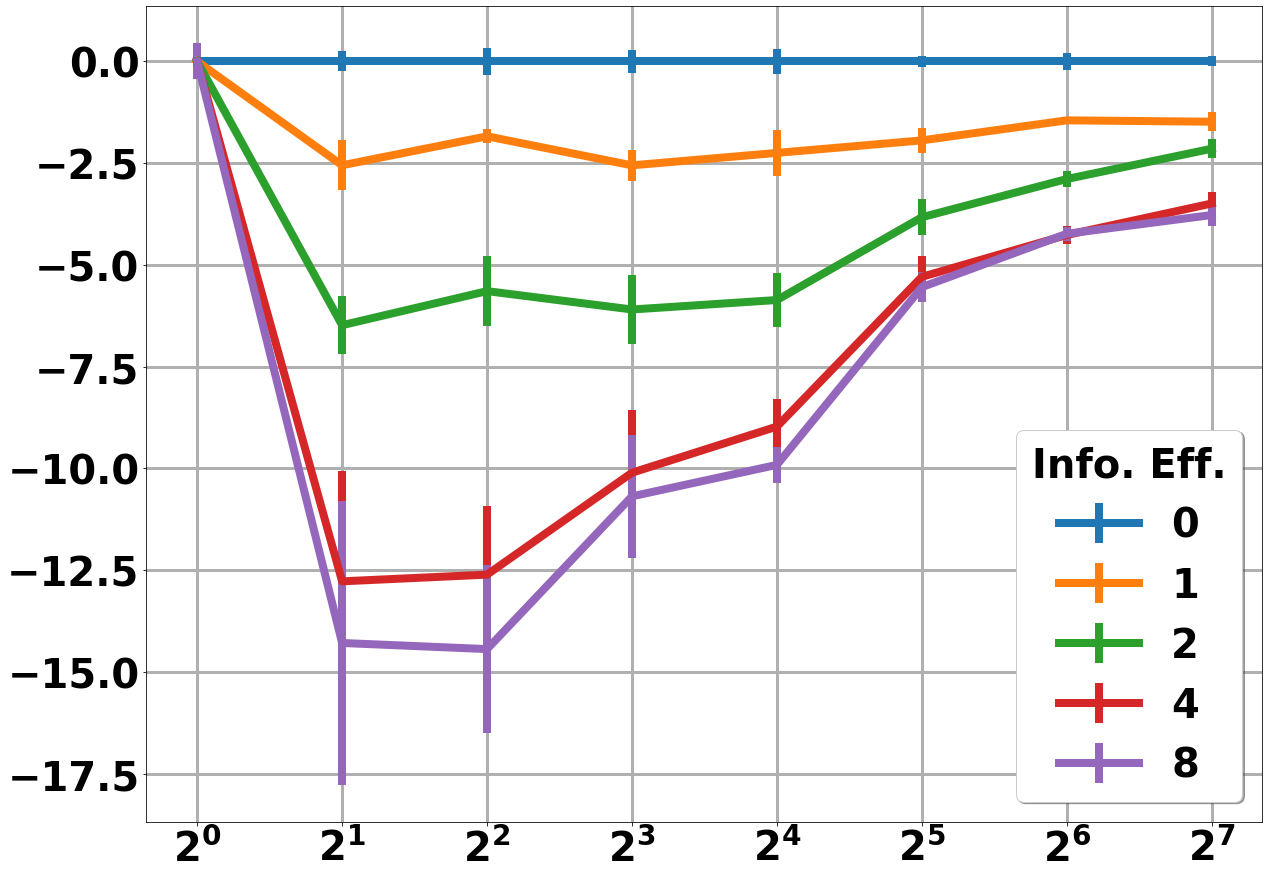}
   \caption{Adult -- Log. Reg.}
   \label{fig:exp4}
  \end{subfigure}
  \begin{subfigure}[c]{0.38\textwidth} 
   \includegraphics[width=\textwidth]{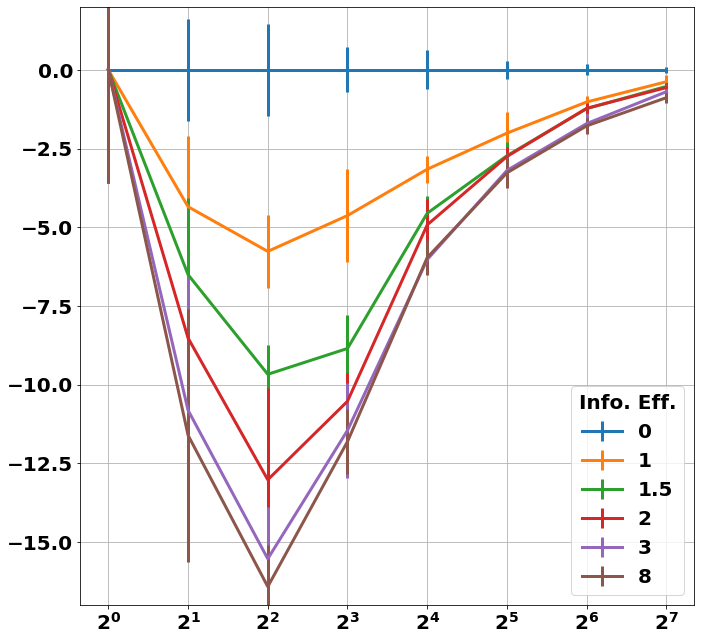}
      \caption{Postures -- Log. Reg. }
   \label{fig:exp5}
  \end{subfigure} 
  \end{center}
  \caption{Figure 3 experiment with Log. Reg.}
  \label{fig:app_log_reg_experiment}
\end{figure}

\begin{figure}[h]
 \begin{center}
  \begin{subfigure}[l]{0.32\textwidth}
    \includegraphics[width=\textwidth]{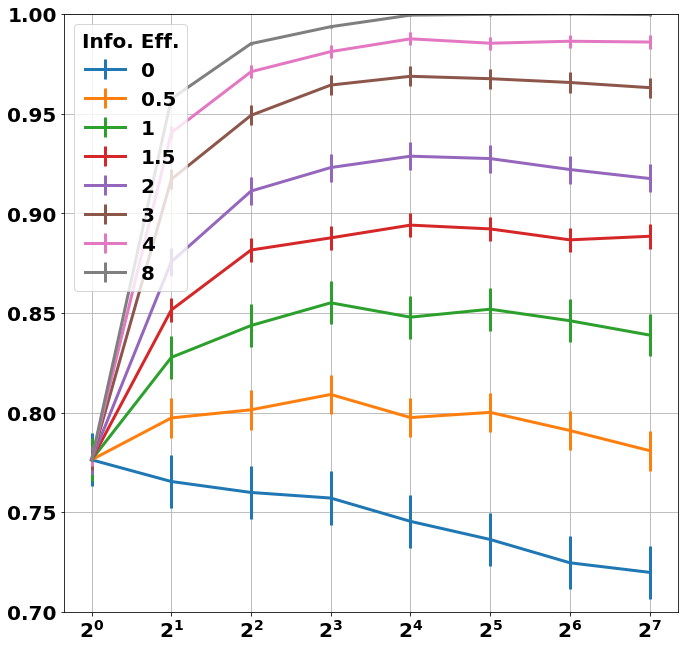}
   \caption{Adult -- NN}
   \label{fig:exp6}
  \end{subfigure}
  \begin{subfigure}[c]{0.32\textwidth} 
   \includegraphics[width=\textwidth]{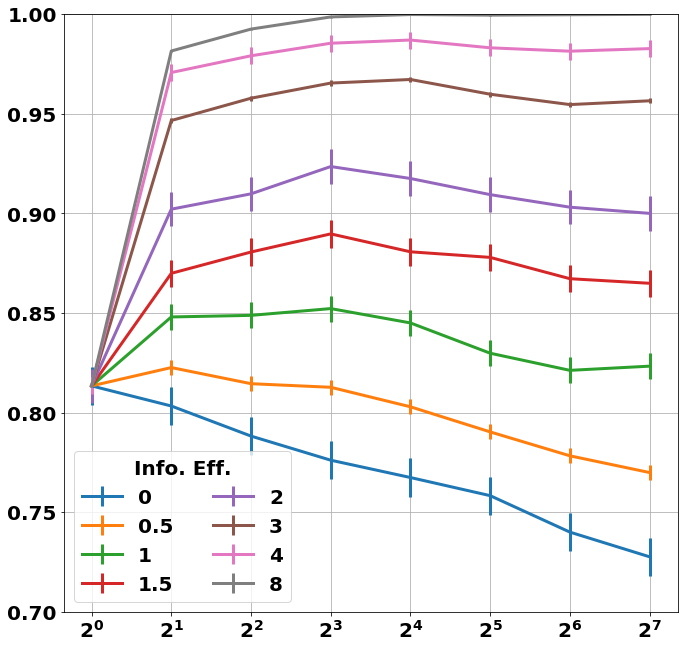}
      \caption{Adult -- Log. Reg. }
   \label{fig:exp7}
  \end{subfigure} 
  \begin{subfigure}[r]{0.32\textwidth}
\includegraphics[width=\textwidth]{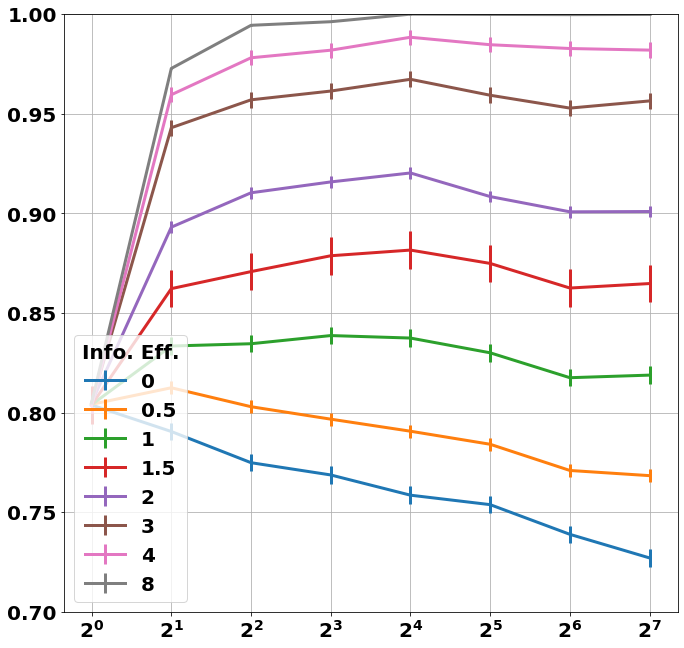}
  \caption{Adult -- MLP}
  \label{fig:exp8}
  \end{subfigure}
  
  \begin{subfigure}[l]{0.32\textwidth}
    \includegraphics[width=\textwidth]{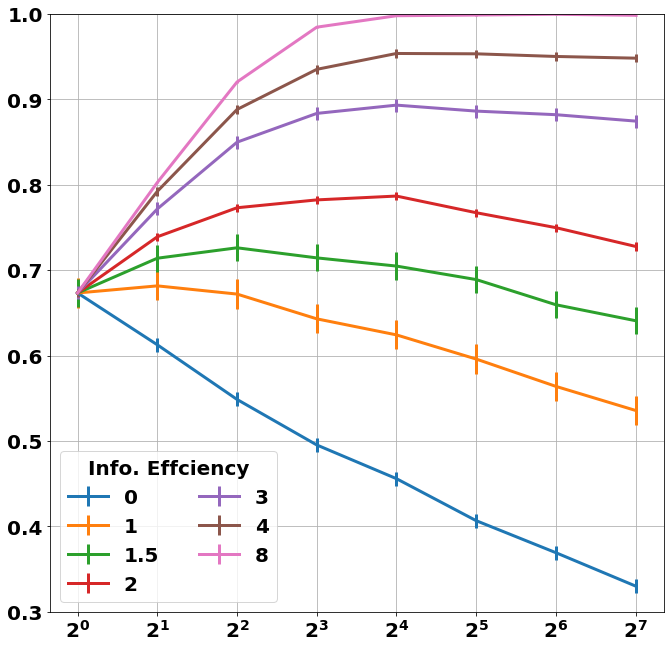}
   \caption{Postures -- NN}
   \label{fig:exp1}
  \end{subfigure}
  \begin{subfigure}[c]{0.32\textwidth} 
   \includegraphics[width=\textwidth]{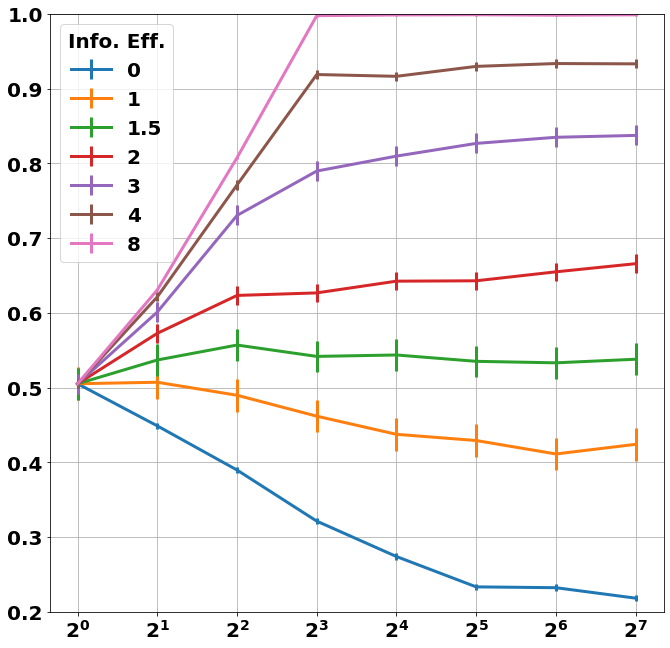}
      \caption{Postures -- Log. Reg.}
   \label{fig:exp2}
  \end{subfigure} 
  \begin{subfigure}[r]{0.32\textwidth}
\includegraphics[width=\textwidth]{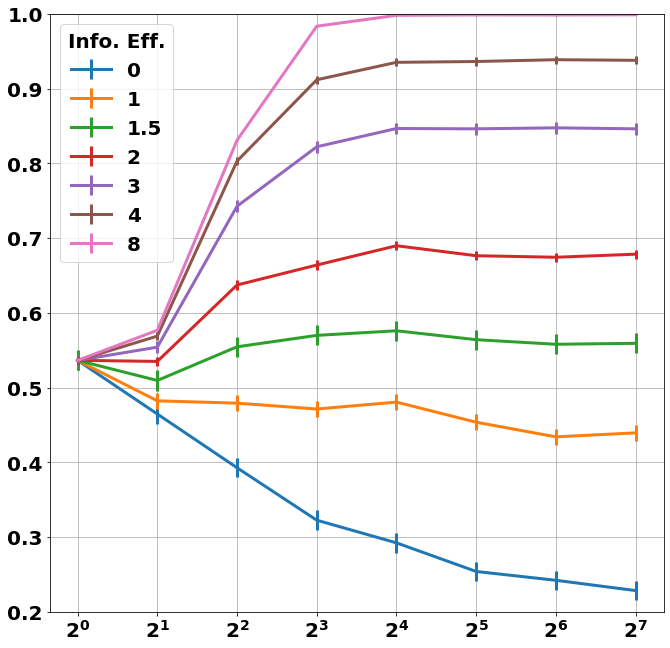}
  \caption{Postures -- MLP}
  \label{fig:exp3}
  \end{subfigure}
  \end{center}
  \caption{Figure 4 experiment with a larger sweep of $\alpha$ reported}
  \label{fig:app_larger_sweep}
\end{figure}
  
\newpage
\section{Mathematical Details}

We include mathematical and theoretical details here. This includes formal definitions for the learning competition as well as proofs of the theorems in Sec. 4.

\subsection{Definitions}

We proceed by formally defining the details of the competition dynamics  which the theorems in the main text use. This particular instance can be generalized to include a wider range of scenarios, including the collaborative filtering simulations. 

\begin{definition} (Supervised Prediction Competition)

Prediction competition  $\mathcal{G} = (\mathcal{D}, \mathbf{A}, \mathbf{SELECT},T)$ where
\begin{enumerate}
    \item $\mathcal{D}$ is a general population distribution over $\mathcal{X} \times \mathcal{Y} \subseteq \mathbb{R}^d \times \{1, \dots, |\mathcal{Y}|\}$ defining a supervised classification task
    \item  $\mathbf{A} = \{A^{(1)}, .., A^{(k)} \}$ is a set of $k$ competing predictors such that each predictor $A^{(i)}$ updates a learning algorithm over time as described in Sec. 2 of the main text.
    \item A selection rule $\mathbf{SELECT}: \mathbb{R}^k \to [k]$ is a function over $\mathbb{R}^k$, \textit{i.e.}, a space of $k$ prediction qualities, and outputs one of the $k$ predictors. 
    \item $T \in \mathbb{N}$ is the number of rounds in the competition.
\end{enumerate}
\end{definition}

When it is unambiguous we may identify a predictor $A^{(i)}$ with their index, $i \in[k]$. As mentioned in Sec. 2, we will use the softmax rule, parameterized by $0 \leq \alpha \leq \infty$ as our user choice model at each round. We use 0-1 loss for $\ell$. 

\subsection{Proofs}

\subsubsection{Proof of Theorem 4.1}

We proceed to restate and prove Thm 4.1. Note that many non-parametric methods are $C$-approximations to the Bayes error rate.  For example, the nearest neighbors method is a $2$-approximation \cite{cover1967nearest}. Here, we use the standard notion of a $C$-approximation ratio in optimization (see \cite{ausiello2012complexity}).

\begin{theorem}[Theorem 4.1]
Suppose users have perfect information  ($\alpha = \infty$) and each predictor is trained using a non-parametric method that is asymptotically a $C$-approximation to the Bayes error rate.  Let $s = |D_0|$ be the number of i.i.d. seed samples that each predictor starts with and assume $s \geq 5$. 
Then there exists $\mathcal{D}$ such that for any $k > 1$, $\lim_{t \rightarrow \infty} \frac{\mathcal{R}^{k}_t}{\mathcal{R}^1_t} = \infty$.

\end{theorem}

\begin{proof}

It will suffice to construct a distribution that can be easily analyzed. Let $P$ denote this distribution. At a high-level, the strategy will be to construct a noise-free ground truth $P$ that results in any particular predictor's expected error rate to be bounded away from $0$. This suffices to complete the claim, since any non-zero expected risk is sufficient.

We proceed to define the distribution $P$. Let $P_X$ be a distribution over $\mathbb{R}$.

\[ P_X(x) = \begin{cases} 
    \frac{1}{s} & x = 1    \\
    1 - \frac{1}{s} &  x = 0\\
      0 & \textbf{else} 
      
   \end{cases}
\]

and the marginal $P_{Y|X}$ is defined by $Y = X$.

Fix an arbitrary predictor. Let $\mathcal{E}$ be the event that this fixed predictor's seed set will lack any points labeled $1$. Then:

\begin{equation}
    \mathbf{Pr}(\mathcal{E}) = \left(1-\frac{1}{s}\right)^s \geq \frac{1}{4}
\end{equation} 

Where the inequality holds for any positive $s$. On the other hand, let $\mathcal{F}$ be the event that at least one other of the $k-1$ remaining predictors do sample a seed set that does contain a point labeled $1$. Then:

\begin{equation}
\mathbf{Pr}(\mathcal{\mathcal{F}}) = 1-\left(1-\frac{1}{s}\right)^{s(k-1)} \geq 1 - e^{-k+1}
\end{equation} 

Since the events are independent, combining them yields:

\begin{equation}
\mathbf{Pr}(\mathcal{\mathcal{F} \cap \mathcal{E}}) \geq \frac{1}{4}(1-e^{-k+1}) 
\end{equation} 

Finally, to complete the argument, we point out that in the event $\mathcal{F} \cap \mathcal{E}$, the fixed predictor will never obtain a sample labeled $1$ because they cannot predict a $1$ due to the fact that they lack any such points in their seed sets. Meanwhile, there exists another predictor who does have a seed point labeled $1$ and will correctly predict all points labeled $1$. Thus, the fixed predictor will never obtain a zero error rate. 
In particular we have that

\begin{equation}
\mathcal{R}_t^{k} \geq \mathbf{Pr}(\mathcal{F} \cap \mathcal{E} \cap \{Y=1\}) =  \frac{1}{4ks}(1-e^{-k+1}) > 0
\end{equation} 

for all $t$. 

Of course, it is unavoidable that the error rate not decay to $0$ for large $s$, but the linear rate of decay in the seed size is also not particularly fast. 

It is clear that because $P_{Y|X}$ is deterministic, a $C$-approximation should also converge to zero in large $t$, but we have shown that $\mathcal{R}_t^k$ is bounded away from $0$ for all $t$ when $k > 1$. This completes the proof.

\end{proof}{}

\subsubsection{Proof of Theorem 4.2}

It is worth clarifying that for the Thm 4.2, proved below, we will be assuming that the nearest-neighbor algorithm uses a majority vote tie-breaking procedure when the nearest-neighbor is non-unique.

\begin{theorem}[Theorem 4.2]
Suppose $k=2$ and both predictors use the nearest-neighbor algorithm. Assume that $\mathbf{Pr}(Y = f(X)) > 1- \epsilon$ for some function $f$ and $\epsilon < \frac{1}{3}$, and $\alpha > \log\frac{1 -\epsilon}{\epsilon}$. Let $s = |D_0|$ be the number of i.i.d. seed samples that each predictor starts with and assume $s \geq 2$. Then there exists $\mathcal{D}$ such that 

$$
\lim_{t \rightarrow \infty} \frac{\mathcal{R}_{t}^2}{\mathcal{R}^{1}_t}  \geq 1 +   \frac{\left(4\epsilon(1-\epsilon)\right)^{s/2}}{9\sqrt{2s}} \left(\frac{1}{2} - \epsilon\right) \left(1-2\frac{1-\epsilon}{1-\epsilon + \epsilon e^{\alpha}}\right)^2
$$

A slightly looser version that removes $\epsilon$ dependence holds whenever $\alpha > \log{(2)}$:

$$
\lim_{t \rightarrow \infty} \frac{\mathcal{R}_{t}^2}{\mathcal{R}^{1}_t}  \geq 1 +  \frac{1}{54\sqrt{2s}}\left(\frac{8}{9\sqrt{s}}  \right)^{s/2}\left(1 - \frac{2}{2 + e^\alpha}\right)^2
$$


\end{theorem}

\begin{proof}
Suppose $X = 0$. Let $f$ be constant with $f(\cdot) = 1$. Let $Y = 1$ with probability $1 - \epsilon$ and $Y = 0$ with probability $\epsilon$ for some $\epsilon < \frac{1}{2}$. This joint distribution $P_{XY}$ satisfies the assumptions. The fact that $Y|X$ is non-deterministic is essential. Without this, as long as $\alpha$ is finite, an interpolating non-parametric method would asymptotically $\varepsilon$-cover the input domain, which would result in no penalty.

We will use the following proof strategy with $P$ as our example distribution. For some event $\mathcal{E}$:

\begin{equation}
    \mathcal{R}_{t}^2 \geq \mathbf{Pr}(\mathcal{E}) \mathcal{R}_t^2|\mathcal{E} + (1-\mathbf{Pr}(\mathcal{E}))\mathcal{R}^{*}
\end{equation}

which implies

\begin{equation}
  \frac{\mathcal{R}_{t}^2}{\mathcal{R}^{*}}  \geq  1 + \frac{ \mathbf{Pr}(\mathcal{E})(\mathcal{R}_t^2|\mathcal{E} - \mathcal{R}^{*}) }{\mathcal{R}^{*}}
\end{equation}

where $\mathcal{R}_t^2|\mathcal{E}$ denotes the expected error rate conditioned on event $\mathcal{E}$ and $\mathcal{R}^*$ is the Bayes error rate. We will also show that under $P$:

\begin{equation}
    \lim_{t \rightarrow \infty} \mathcal{R}_t^1 \rightarrow \mathcal{R}^* = \epsilon
\end{equation}

Combining (6) and (7) will then yield the claim. To show (7), we can simply note that the fraction of $1$s in the predictor's data set will concentrate around $1-\epsilon$ due to the law of large numbers \cite{hsu1947complete}. Since this will be the majority, the majority tie-breaker will be in effect, meaning that the predictor will always predict $1$. This prediction is the Bayes optimal prediction. This implies (7) holds.

To finish the proof, we must give expressions or bounds for $\mathbf{Pr}(\mathcal{E})$ and $\mathcal{R}_t^2|\mathcal{E}$ under a suitably defined event $\mathcal{E}$.


Let $Q_0(t)$ and $Q_1(t)$ be the fraction of $1$s in the data sets for each predictor, respectively, at round $t$. Thus, $Q_i(0)$ is the fraction of $1$s in the seed set of agent $i$.

\begin{equation}
    sQ_i(0) \sim \mathbf{Bin}(s, 1-\epsilon)
\end{equation}

We can use the following bounds on the deviations of a Binomial \cite{arratia1989tutorial,matouvsek2001probabilistic,robert1990ash} to yield the following bounds:

\begin{equation}
  \exp{\left(-s D(\frac{1}{2}||1-\epsilon)   \right) } \geq  \mathbf{Pr}(Q_i(0) < \frac{1}{2} ) \geq \frac{1}{\sqrt{2s}}\exp{\left(-s D(\frac{1}{2}||1-\epsilon)   \right) }
\end{equation}

Where $i \in \{0,1\}$ and $D(a||b)$ is the binary relative entropy, given by

\begin{equation}
     D(a||b) = a \log(\frac{a}{b}) + (1-a)\log(\frac{1-a}{1-b})
\end{equation}    

Using elementary logarithmic identities we can simplify:

\begin{equation}
  \exp{\left(-s D(\frac{1}{2}||1-\epsilon)   \right) } = (4\epsilon(1-\epsilon))^{s/2}
\end{equation}

 The tail bound inequalities simplify into:

\begin{equation}
  \sqrt{2s}\mathfrak{B} \geq  \mathbf{Pr}(Q_i(0) < \frac{1}{2}) \geq \mathfrak{B}
\end{equation}

Where, as a shorthand, we let $\mathfrak{B} =  \frac{(4\epsilon(1 - \epsilon))^\frac{s}{2}}{\sqrt{2s}}$.

We proceed to define a suitable $\mathcal{E}$. In order to do so, we define event $\mathcal{F}(\tau)$, parameterized by $0 \leq \tau \leq \infty$, as follows:   

\begin{equation}
    \mathcal{F(\tau)} = \left( \bigcap_{0 \leq l < \tau} \{Q_0(l) < \frac{1}{2} < Q_1(l)\}\right) \cup \left( \bigcap_{0 \leq l < \tau} \{Q_1(l) < \frac{1}{2} < Q_0(l)\}\right)
\end{equation}

Let $\iota = \argmin_{i}\{Q_i(0)\}$ and $\bar{\iota} = \argmax_{i}\{Q_i(0)\}$. Let $\mathcal{E} = \mathcal{F}(\infty)$. Recall that  $\hat{Y}_{t}^{(j)}$ is agent $j$'s prediction at time $t$.  Observe that $\mathcal{F}(t)$ implies both  $\hat{Y}_{t}^{(\iota)} = 0$ and $\hat{Y}_{t}^{(\bar\iota)} = 1$ due to the majority rule. With $\mathcal{E}$ now defined, we proceed to concoct a bound for $\mathbf{Pr}(\mathcal{E})$.

 In event $\mathcal{F}(0)$, it must be the case that predictor $\iota$ has at least one fewer $1$s than $0$s in the seed set and predictor $\bar\iota$ has at least one more $1$ than $0$. En route for our bound on $\mathbf{Pr}(\mathcal{E})$ we derive the lower bound for $\mathbf{Pr}(\mathcal{F}(0))$ as follows. Notice that $\mathcal{F}(0)$ can be easily expressed in terms of $\{Q_i(0) \leq \frac{1}{2})\}$ which makes $\mathfrak{B}$ an ideal expression for bounding $\mathbf{Pr}(\mathcal{F}(0))$.

\begin{equation}
    \mathbf{Pr}(\mathcal{F}(0)) \geq \mathfrak{B}(1-\sqrt{2s}\mathfrak{B})
\end{equation}

In turn, this yields us the following:

\begin{equation}
    \mathbf{Pr}(\mathcal{E}) \geq \mathbf{Pr}(\mathcal{E}|\mathcal{F}(0))\mathbf{Pr}(\mathcal{F}(0)) \geq \mathbf{Pr}(\mathcal{E}|\mathcal{F}(0))\mathfrak{B}(1-\sqrt{2}\mathfrak{B})
\end{equation}
Thus, we are left with the task of bounding $\mathbf{Pr}(\mathcal{E}|\mathcal{F}(0))$ to find an expression to bound $\mathbf{Pr}(\mathcal{E})$. We proceed by using random walk theory (\cite{feller2008introduction,spitzer2013principles} are suitable references for the uninitiated).  We will study an integer-valued stochastic process over the integers. At time $\mathfrak{t} \in \mathbb{N}$, $\mathfrak{X}_{\mathfrak{t}} \in \mathbf{Z}$. Furthermore, $\mathfrak{X}_{t+1} \in \{\mathfrak{X}_{t}-1, \mathfrak{X}_{t}, \mathfrak{X}_t +1 , \mathfrak{} \} $. The distribution over the increments is defined by $(\mathfrak{q},\mathfrak{p}, \mathfrak{r})$ as follows:

\begin{equation}
    \mathbf{Pr}(\mathfrak{X}_{\mathfrak{t}+1} = \mathfrak{X}_\mathfrak{t} - 1) =  \mathfrak{q} \text{     and     }  \mathbf{Pr}(\mathfrak{X}_{\mathfrak{t}+1} = \mathfrak{X}_\mathfrak{t} + 1) = \mathfrak{p} \text{ and } \mathbf{Pr}(\mathfrak{X}_{\mathfrak{t}+1} = \mathfrak{X}_\mathfrak{t} ) = \mathfrak{r}
\end{equation}

We denote the random walk distribution by $\mathbf{RW}(q,p,r)$ and write $\{\mathfrak{X}\}_{\mathfrak{t}=0}^{\infty} \sim \mathbf{RW}(\mathfrak{q},\mathfrak{p}, \mathfrak{r})$ to associate the random variable to the distribution. We will also define and use an independent copy of the random walk, denoted by $\{\mathfrak{X}'\}_{\mathfrak{t}=0}^{\infty} \sim \mathbf{RW}(\mathfrak{q'},\mathfrak{p'}, \mathfrak{r}')$. The analysis of random walks is a rich subject and many techniques are known for computing various probabilities for events of interest. Of particular utility here is the fact that if $\mathfrak{X}_0 = 1$ and $\frac{\mathfrak{q}}{\mathfrak{q} + \mathfrak{p}} < \frac{1}{2}$, then the probability that the walk never reaches the origin is given  by

\begin{equation}
    \mathbf{Pr}(\mathfrak{X}_{\mathfrak{t}} > 0, \forall \mathfrak{t}) = 1 - \frac{2\mathfrak{q}}{\mathfrak{q} + \mathfrak{p}}
\end{equation}

This result that can be obtained by solving a recurrence relation, as in \cite{feller2008introduction,spitzer2013principles}. Notice that the probability that a random walk with bias $(q,p,r)$ ever reaches the origin is the same as a random walk with bias  $(\frac{q}{q + p},\frac{p}{q + p},0)$. We will assume that the random walks start at 1: $\mathfrak{X}_{0} = \mathfrak{X'}_0 = 1$.

For convenience, we will define events $\mathcal{H}(\tau)$ as follows.

\begin{equation}
     \mathcal{H(\tau)} = \bigcap_{0 \leq l \leq \tau} \left( \{\mathfrak{X}_l > 0 \} \cap \{\mathfrak{X'}_l > 0 \} \right)
\end{equation}

As we will see, our target $\mathbf{Pr}(\mathcal{E}|\mathcal{F}(0))$ can be bounded by the probability that the random walks in question never reach the origin:  $\mathbf{Pr}(\mathcal{H}(\infty)) = \mathbf{Pr}(\mathfrak{X}_{\mathfrak{t}} > 0, \mathfrak{X'}_{\mathfrak{t}} > 0, \forall \mathfrak{t})$.
To see how this works, we begin by considering the following choices for $(\mathfrak{q}, \mathfrak{p}, \mathfrak{r})$ and $(\mathfrak{q}, \mathfrak{p}, \mathfrak{r})$:

  \begin{equation}
     (\mathfrak{q}, \mathfrak{p}, \mathfrak{r})  = \left(  \frac{\epsilon}{1+e^\alpha}, \frac{(1-\epsilon) e^\alpha}{1+e^\alpha},  \frac{1-\epsilon + \epsilon e^\alpha}{1+e^\alpha}  \right)
 \end{equation}
 
  \begin{equation}
     (\mathfrak{q'}, \mathfrak{p'}, \mathfrak{r'})  = \left(  \frac{ 1 - \epsilon}{1+e^\alpha}, \frac{\epsilon e^\alpha}{1+e^\alpha}, \frac{(1-\epsilon)e^\alpha + \epsilon}{1+e^\alpha} \right)
 \end{equation}
 
 Also note that it is easy to give $\mathbf{Pr}(\mathcal{H}(\infty))$ in terms of $(\mathfrak{q}, \mathfrak{p}, \mathfrak{r})$ and $(\mathfrak{q}', \mathfrak{p'}, \mathfrak{r'})$  by noting that $\mathbf{Pr}(\mathcal{H}(\infty)) =  \mathbf{Pr}(\mathfrak{X}_{\mathfrak{t}} > 0, \forall \mathfrak{t})\mathbf{Pr}(\mathfrak{X'}_{\mathfrak{t}} > 0, \forall \mathfrak{t})$.
 
  \begin{equation}
   \mathbf{Pr}(\mathcal{H}(\infty)) =  \left(1 - \frac{2\mathfrak{q}}{\mathfrak{q} + \mathfrak{p}}\right)\left(1 - \frac{2\mathfrak{q'}}{\mathfrak{q'} + \mathfrak{p'}}\right)
 \end{equation}

 The reason for this choice will be made apparent later. By assumption we know that $\alpha$ is not too small and $\epsilon$ is not too large. Precisely, the conditions are: $\alpha > \log\frac{1-\epsilon}{\epsilon}$ and $\epsilon < \frac{1}{2}$. With elementary algebra one can verify these conditions are sufficient to imply that $\frac{\mathfrak{q}}{\mathfrak{q} + \mathfrak{p}} < \frac{1}{2}$ and $\frac{\mathfrak{q'}}{\mathfrak{q'} + \mathfrak{p'}} < \frac{1}{2}$.

We proceed to construct a relationship between $\mathfrak{X}_t$ and $Q_{\bar\iota}(t)$ and between $\mathfrak{X'}_t$ and $Q_{\iota}(t)$. We introduce $\mathbf{S}(t) = \mathbf{SELECT}(t)$ as a shorthand notation.   Let $ \mathbf{1}(\cdot)$ denote the indicator function. Consider $W_t$ and $W'_t$ defined as below:

\begin{equation}
    W_t  = \left(s + \sum_{l=1}^{t}\mathbf{1}(\mathbf{S}(l) = \bar\iota) \right) \left(2Q_{\bar\iota}(t) - 1\right)
\end{equation}
\begin{equation}
    W'_t  = \left(s + \sum_{l=1}^{t}\mathbf{1}(\mathbf{S}(l) = \iota) \right) \left(1 - 2Q_{\iota}(t) \right)
\end{equation}

Intuitively, $W$ is renormalizing $Q_{\bar\iota}$ from $\frac{\# \text{ of 1s}}{\text{total } \#}$ to $(\# \text{ of 1s}) - (\# \text{ of 0s})$ (and analogously for $W'_t$ and $Q_\iota$). Under this transform, we can express $\mathcal{F}(\tau)$ in terms of $W$ and $W'$ as follows:

\begin{equation}
    \mathcal{F(\tau)} =  \bigcap_{0 \leq l < \tau}  ( \{W_l > 0\} \cap \{W_l' > 0 \})
\end{equation}

Because, of course, $W_t > 0$ if and only if $Q_{\bar\iota} > \frac{1}{2}$ (and analogously for $W'_t$ and $Q_\iota$).  Recalling that $\mathbf{SELECT}$ uses the softmax rule and with Bayes rule \cite{feller2008introduction, grimmett2014probability} we can obtain the following expressions for the distribution over the increments to $W$ and $W'$.  Notice that when conditioning on $\mathcal{F}(t-1)$, we know that $\hat{Y}_{t}^{(\iota)} = 0$ because learner $\iota$'s dataset must have a majority of $0$s. Similarly, $\hat{Y}_{t}^{(\bar\iota)} = 1$.

\begin{equation}
    \mathbf{Pr}(W_{t} =  W_{t-1} - 1 |\mathcal{F}(t-1) ) =  \mathbf{Pr}(\mathbf{S}(t) = \bar\iota,  Y_t = 0|\mathcal{F}(t-1)) = \frac{\epsilon}{1+e^\alpha} = \mathfrak{q}
\end{equation}
\begin{equation}
    \mathbf{Pr}(W_{t} =  W_{t-1} + 1 |\mathcal{F}(t-1) ) =  \mathbf{Pr}(\mathbf{S}(t) = \bar\iota,  Y_t = 1|\mathcal{F}(t-1)) = \frac{(1-\epsilon) e^\alpha}{1+e^\alpha} = \mathfrak{p}
\end{equation}
\begin{equation}
    \mathbf{Pr}(W_{t} =  W_{t-1} |\mathcal{F}(t-1) ) =  \mathbf{Pr}(\mathbf{S}(t) = \iota  |\mathcal{F}(t-1)) = \frac{1-\epsilon + \epsilon e^\alpha}{1+e^\alpha} = \mathfrak{r}
\end{equation}

\begin{equation}
    \mathbf{Pr}(W'_{t} =  W'_{t-1} - 1 |\mathcal{F}(t-1) ) =  \mathbf{Pr}(\mathbf{S}(t) = \iota,  Y_t = 1|\mathcal{F}(t-1)) = \frac{ 1 - \epsilon}{1+e^\alpha} = \mathfrak{q'}
\end{equation}
\begin{equation}
    \mathbf{Pr}(W'_{t} =  W'_{t-1} + 1 |\mathcal{F}(t-1) ) =  \mathbf{Pr}(\mathbf{S}(t) = \iota,  Y_t = 0|\mathcal{F}(t-1)) = \frac{\epsilon e^\alpha}{1+e^\alpha} = \mathfrak{p'}
\end{equation}

\begin{equation}
    \mathbf{Pr}(W'_{t} =  W'_{t-1} |\mathcal{F}(t-1) ) =  \mathbf{Pr}(\mathbf{S}(t) = \bar\iota  |\mathcal{F}(t-1)) = \frac{(1-\epsilon)e^\alpha + \epsilon}{1+e^\alpha} = \mathfrak{r'}
\end{equation}

This reveals the reasoning behind the choices for  $(\mathfrak{q}, \mathfrak{p}, \mathfrak{r})$ and $(\mathfrak{q'}, \mathfrak{p'}, \mathfrak{r'})$.  We proceed to bound the increments of $W$ and $W'$ by using $\mathfrak{X}$ and $\mathfrak{X'}$,  but this requires a bit more work. Observe that $(W,W')_t$ are coupled whereas $(\mathfrak{X}$, $\mathfrak{X'})_t$ are independent. The trick will be the couple $(W,W')_t$ and $(\mathfrak{X}$, $\mathfrak{X'})_t$ in a prudent way. To this end, we will introduce a third random process, $(Z, Z')_t$. First, define $\bar\rho(t)$ and $\rho(t)$ as follows:

\begin{equation}
    \bar\rho(t) =  \min_{m \in S_t(\{\mathfrak{X}_j\}) } m 
\end{equation}

\begin{equation}
    \rho(t) =  \min_{m \in S_t(\{\mathfrak{X'}_j\}) } m 
\end{equation}

where $S_t$ is a set-valued function  defined over sequences $\{x_j\}_{j=0}^\infty$ as:

\begin{equation}
 S_t(\{x_j\}) = \{ i : i \in \mathbb{N}, \sum_{j=1}^i \mathbf{1}(x_j \neq x_{j-1}) \geq t \}
\end{equation}

where $\mathbb{N}$ is the set of natural numbers (including zero). Intuitively, $\bar\rho(t)$ is computing the index of process $\mathfrak{X}$ that corresponds to $t$-th non-zero increment (and analogously for $\rho(t)$ and $\mathfrak{X}'$). As for $Z_t$ and $Z'_t$:

\begin{equation}
    Z_t = \mathfrak{X}_{\bar\rho(t)}
\end{equation}

\begin{equation}
    Z'_t = \mathfrak{X'}_{\rho(t)}
\end{equation}

Intuitively $Z_t$ corresponds to the sequence one would obtain from $\mathfrak{X}_t$ with the zero increments deleted (and analogously for $Z'_t$ and $\mathfrak{X'}_t$). Let 

\begin{equation}
    \mathcal{Z(\tau)} =  \bigcap_{0 \leq l < \tau}  ( \{Z_l > 0\} \cap \{Z_l' > 0 \})
\end{equation}

Any $\mathfrak{X}_t$ sequence that reaches the origin in finite time will do so with the non-zero increments deleted, which implies that $Z_t$ will also reach the origin (and analogously for $Z'_t$ and $\mathfrak{X'}_t$). This implies the key fact that event $\mathcal{Z}(\infty)$ occurs  if $\mathcal{H}(\infty)$ occurs. Also, $\mathcal{H}(\infty)$ occurs almost surely if  $\mathcal{Z}(\infty)$ occurs.  Events $\mathcal{Z}(\infty)$ and $\mathcal{H}(\infty)$ are equivalent up to a set of measure zero. Namely, the measure zero event that either $\mathfrak{X}_t$ or $\mathfrak{X'}_t$ produces a finite number of non-zero increments.

\begin{equation}
    \mathbf{Pr}(\mathcal{Z(\infty)}) =     \mathbf{Pr}(\mathcal{H(\infty)})
\end{equation}

It remains to couple $W_t$ with $Z_t$ and $W'_t$ with $Z'_t$. We do so as follows by defining $\bar\psi(t)$ and $\psi(t)$

\begin{equation}
    \bar\psi(t) \sim \mathbf{Bin}\left(t, \mathfrak{r}\right)
\end{equation}

\begin{equation}
    \psi(t) = t - \bar\psi(t)
\end{equation}

\begin{equation}
   W_t = Z_{\bar\psi(t)} + W_0 - 1
\end{equation}

\begin{equation}
   W'_t = Z'_{\psi(t)} + W'_0  - 1
\end{equation}

Notice that because $\mathfrak{r} + \mathfrak{r'} = 1$, the marginals are preserved which makes this a valid coupling. Also notice that the event $\mathcal{Z}(\infty)$ implies $\mathcal{F}(\infty)|\mathcal{F}(0)$ outside of the measure zero event that either $\bar\psi(t)$ or $\psi(t)$ remain bounded as $t \rightarrow \infty$:

\begin{equation}
    \mathbf{Pr}(\mathcal{F}(\infty) | \mathcal{F}(0) ) \geq \mathbf{Pr}( \mathcal{Z}(\infty))
\end{equation}

\begin{equation}
    \mathbf{Pr}(\mathcal{E}|\mathcal{F}(0)) = \mathbf{Pr}(\mathcal{F}(\infty)|\mathcal{F}(0)) \geq \mathbf{Pr}( \mathcal{Z}(\infty)) = \mathbf{Pr}( \mathcal{H}(\infty))
\end{equation}

 We proceed to establish the bound on  $\mathbf{Pr}(\mathcal{E}|\mathcal{F}(0))$. Consider the following: 

\begin{equation}
    \mathbf{Pr}(\mathcal{E}|\mathcal{F}(0)) = \mathbf{Pr}(\mathcal{F}(\infty)|\mathcal{F}(0)) \geq  \mathbf{Pr}(\mathcal{H}(\infty)) =  \left(1 - \frac{2\mathfrak{q}}{\mathfrak{q} + \mathfrak{p}}\right)\left(1 - \frac{2\mathfrak{q'}}{\mathfrak{q'} + \mathfrak{p'}}\right)
\end{equation}

Thus, we can combine our bounds for $\mathbf{Pr}(\mathcal{F}(0))$ and $\mathbf{Pr}(\mathcal{E}|\mathcal{F}(0))$ to obtain a bound on $\mathbf{Pr}(\mathcal{E})$: 

\begin{equation}
    \mathbf{Pr}(\mathcal{E}) \geq \mathfrak{B}(1-\sqrt{2s}\mathfrak{B}) \left(1 - \frac{2\mathfrak{q}}{\mathfrak{q} + \mathfrak{p}}\right)\left(1 - \frac{2\mathfrak{q'}}{\mathfrak{q'} + \mathfrak{p'}}\right)
\end{equation}

It now remains to get a bound for $\mathcal{R}_t^2|\mathcal{E}$. Of course, this is easy now that we have ascertained that in event $\mathcal{E}$, we know that $\iota$ ($\bar\iota$)  always has a majority $0$s ($1$s) thus always predicts $0$ ($1$). From this it follows immediately by taking an average that

\begin{equation}
    \mathcal{R}_t^2|\mathcal{E} = \frac{1}{2} \text{ } \forall t \text{ } \implies  \lim_{t \rightarrow \infty}  \mathcal{R}_t^2|\mathcal{E} = \frac{1}{2}
\end{equation}

Combining the bounds gives:

\begin{equation}
    \lim_{t \rightarrow \infty} \frac{\mathcal{R}_{t}^2}{\mathcal{R}^{1}_t}  \geq  1 + \frac{ \mathbf{Pr}(\mathcal{E})(\mathcal{R}_t^2|\mathcal{E} - \mathcal{R}^{*}) }{\mathcal{R}^{*}}\geq 1+  \frac{1}{\epsilon}\mathfrak{B}(1-\sqrt{2s}\mathfrak{B}) \left(1 - \frac{2\mathfrak{q}}{\mathfrak{q} + \mathfrak{p}}\right)\left(1 - \frac{2\mathfrak{q'}}{\mathfrak{q'} + \mathfrak{p'}}\right) \left(\frac{1}{2} - \epsilon\right)
\end{equation}

At the expense of a looser bound, we can simplify the bound by a series of additional approximations in order to obtain an interpretable expression:

\begin{equation}
  \left(1 - \frac{2\mathfrak{q}}{\mathfrak{q} + \mathfrak{p}}\right)\left(1 - \frac{2\mathfrak{q'}}{\mathfrak{q'} + \mathfrak{p'}}\right) \geq \left(1 - \frac{2\mathfrak{q'}}{\mathfrak{q'} + \mathfrak{p'}}\right)^2 = \left(1-2\frac{1-\epsilon}{1-\epsilon + \epsilon e^{\alpha}}\right)^2
\end{equation}

Which follows because  $\frac{\mathfrak{q'}}{\mathfrak{q'} + \mathfrak{p'}} > \frac{\mathfrak{q}}{\mathfrak{q} + \mathfrak{p}} $.

We can also crudely simplify the $\mathfrak{B}$-term as follows. Recall:

\begin{equation}
    \mathfrak{B} = \frac{(4\epsilon(1-\epsilon))^{s/2}}{\sqrt{2s}}
\end{equation}

With the assumption that $\epsilon \leq 1/3$ we note that the following holds:

\begin{equation}
    \max_{0 \leq \epsilon \leq 1/3}\{4\epsilon(1-\epsilon) \} = 8/9
\end{equation}

As well as the following identity:

\begin{equation}
    (4\epsilon(1-\epsilon)) \geq \epsilon \text{ for } 0 \leq \epsilon \leq 1/3
\end{equation}

Together, we can bound $\mathfrak{B}$ from both above and below:

\begin{equation}
   \frac{(8/9)^{s/2}}{\sqrt{2s}} \geq  \mathfrak{B} \geq  \frac{\epsilon^{s/2}}{\sqrt{2s}} 
\end{equation}

Allowing us to conclude:

\begin{equation}
1-\sqrt{2s}\mathfrak{B} \geq 1/9
\end{equation}

And,

\begin{equation}
 \mathfrak{B}(1-\sqrt{2s}\mathfrak{B}) \geq \frac{\mathfrak{B}}{9} \geq  \frac{\epsilon^{s/2}}{9\sqrt{2s}} 
\end{equation}




Thus, we arrive at a neater expression in terms of $\epsilon$, $s$, and $\alpha$:

\begin{equation}
\lim_{t \rightarrow \infty} \frac{\mathcal{R}_{t}^2}{\mathcal{R}^{1}_t}  \geq 1 +   \frac{\left(4\epsilon(1-\epsilon)\right)^{s/2}}{9\sqrt{2s}} \left(\frac{1}{2} - \epsilon\right) \left(1-2\frac{1-\epsilon}{1-\epsilon + \epsilon e^{\alpha}}\right)^2
\end{equation}

\begin{equation}
 \geq 1 +   \frac{\epsilon^{s/2}}{9\sqrt{2s}} \left(\frac{1}{2} - \epsilon\right) \left(1-2\frac{1-\epsilon}{1-\epsilon + \epsilon e^{\alpha}}\right)^2
\end{equation}

Finally, due to monotonicity of the expression in $\epsilon$ we can remove the dependence on $\epsilon$ by fixing  $\epsilon = \frac{1}{3}$:

\begin{equation}
\lim_{t \rightarrow \infty} \frac{\mathcal{R}_{t}^2}{\mathcal{R}^{1}_t}  \geq 1 +  \frac{1}{54\sqrt{2s}}\left(\frac{8}{9\sqrt{s}}  \right)^{s/2}\left(1 - \frac{2}{2 + e^\alpha}\right)^2
\end{equation}

\end{proof}

\subsubsection{Proof of Theorem 4.3}

It  will be helpful to separately prove a lemma for use in Theorem 4.3's proof. This lemma upper bounds the variance of a symmetrical truncated Binomial with the variance of a usual Binomial with the same number of trials as support left in the truncated Binomial. To clarify notion in the proof of the lemma, note that we define $f(x)$ $\propto g(x)$ to mean:

\begin{equation}
    f(x) = C g(x) \text{ } \forall x \in \mathcal{X}
\end{equation}

where $C$ is some fixed constant independent of $x$ and the equation holds over all choices of $x$ in some set $\mathcal{X}$ which can be inferred from context.

\begin{lemma}
\label{lemma:binvar}
If $X \sim \mathbf{Bin}(2n, \frac{1}{2})$, then  for any integer $c$ such that $n >  c > 0$:

$$\mathbf{Var}(X|n-c \leq X \leq n+c) \geq \mathbf{Var}(\mathbf{Bin}(2c,1/2)) = \frac{c}{2}$$

\end{lemma}
\begin{proof}
As a notional shorthand, we will use $X_{\mathbf{tr}}$ to denote the truncated version of $X$: 

\begin{equation}
    X_{\mathbf{tr}} \sim X|\{n - c \leq X \leq n + c \}
\end{equation}

By the definition of truncation, $\mathbf{Pr}(X_{\mathbf{tr}} = x) \propto \mathbf{Pr}(X = x)$ for all $x$ in the truncated support. From this, it follows that $X_{\mathbf{tr}}$ inherits symmetry and unimodality from $X$.  Furthermore, from the Binomial pmf it also directly follows that $\mathbf{Pr}(X = x) \propto \binom{n}{x}$ because $p = \frac{1}{2}$. Let $X_{\mathbf{sup}}$ denote the usual Binomial defined over the truncated support: 

\begin{equation}
    X_{\mathbf{sup}} \sim \mathbf{Bin}(2c,1/2) + n -c 
\end{equation}

The constant translation is just an aesthetic to keep the supports of $X_{\mathbf{sup}}$ and $X_{\mathbf{tr}}$ identical. Observe that for both $X_{\mathbf{sup}}$ and $X_{\mathbf{tr}}$, $n$ is both the mean and the mode outcome (keep in mind the symmetry and unimodality of both). We will complete the proof by showing that:

\begin{equation}
    1 \geq \frac{\mathbf{Pr}(X_{\mathbf{tr}} = n + j)}{\mathbf{Pr}(X_{\mathbf{tr}} = n)} \geq  \frac{\mathbf{Pr}(X_{\mathbf{sup}} = n + j)}{\mathbf{Pr}(X_{\mathbf{sup}} = n)} \text{  for all $j$ in the support}
\end{equation}

Notice that this inequality would imply that $X_{\mathbf{tr}}$ is strictly less concentrated that $X_{\mathbf{sup}}$, and since they have the same support, are unimodal, and are symmetric, it follows that $X_{\mathbf{tr}}$ has higher variance. We proceed to demonstrate this key inequality.

Based on our established proportionality rules we can express the inequality in terms of factorials:

\begin{equation}
    \frac{\mathbf{Pr}(X_{\mathbf{tr}} = n + j)}{\mathbf{Pr}(X_{\mathbf{tr}} = n)} = \frac{\binom{2n}{n+j}}{\binom{2n}{n}} = \frac{(n!)^2}{(n+j)!(n-j)!} = \prod_{i=1}^j \frac{n-i}{n+i}
\end{equation}

\begin{equation}
    \frac{\mathbf{Pr}(X_{\mathbf{sup}} = n + j)}{\mathbf{Pr}(X_{\mathbf{sup}} = n)} = \frac{\binom{2c}{c+j}}{\binom{2c}{c}}  \frac{(c!)^2}{(c+j)!(c-j)!} = \prod_{i=1}^j \frac{c-i}{c+i}
\end{equation}

In order to compare these quantities, consider the following:

\begin{equation}
   \prod_{i=1}^j \frac{c+1-i}{c+1+i} > \prod_{i=1}^j \frac{c-i}{c+i}
\end{equation}

In order to see the correctness of the above inequality, consider any particular term in the product:

\begin{equation}
   \frac{c+1-i}{c+1+i} > \frac{c-i}{c+i}
\end{equation}

\begin{equation}
  (c+1-i)(c+i) > (c+1+i)(c-i)
\end{equation}

\begin{equation}
  2i > 0
\end{equation}

From which we can conclude that the inequality over the entire product:  $ \prod_{i=1}^j \frac{c+1-i}{c+1+i} > \prod_{i=1}^j \frac{c-i}{c+i}$ must hold since it holds over each term and all are positive. By inductively applying the product inequality, we can then conclude the key inequality because $n > c$:

\begin{equation}
    \prod_{i=1}^j \frac{n-i}{n+i} \geq \prod_{i=1}^j \frac{c-i}{c+i}
\end{equation}

\end{proof}

\begin{theorem}[Theorem 4.3]
Suppose the data is generated from a  linear model $Y = XW + \epsilon$ with $\mathbf{E}(\epsilon|X) = 0$. 
Assume each predictor is uses an ordinary least-squares linear regression.  Let $ s \geq 1$ be the number of i.i.d. seed samples each predictor starts with. We have the following:

(i) If $\alpha > 0$ and $k \geq 2$ then  $\lim_{t\rightarrow \infty} \sup_{\mathcal{D}} \frac{\mathcal{R}^{k}_t}{\mathcal{R}^1_t}  \geq 1 + \frac{1}{3567s^{3/2}}  $ 

(ii) If $\alpha = \infty$ then $\lim_{t \rightarrow \infty} \sup_{\mathcal{D}} \frac{\mathcal{R}^{k}_t}{\mathcal{R}^1_t}  \geq  \frac{2 k}{k+1}$
\end{theorem}

\begin{proof}

For both parts it will be helpful to recall the seminal result of White which shows the consistency of the OLS estimate for linear models \cite{white1980heteroskedasticity}. Thus, $\mathcal{R}^1_t$ converges to the minimum mean square error.

\paragraph{Proof of (ii)} We first prove part (ii). It is sufficient to construct a $P_{XY}$ satisfying the linear model assumptions that can be easily analyzed. In this proof, we construct $P_{XY}$ such that with perfect information and OLS updates, the learning dynamics reduce into the \emph{sequential K-means} dynamics studied in \cite{macqueen1967some}. To be more precise, the resultant random process over the tuple of each predictor's OLS weight estimates is almost the same as that over the tuple of centroids in sequential $K$-means, in a sense that will be made formal in this proof. First we restate the notion of a sequential $K$-means process, originally defined in \cite{macqueen1967some}.

\begin{definition} (Sequential $K$-means process)

Let $P$ be a non-atomic distribution over $\mathbb{R}$ with bounded first and second moments. Let $V(t) \in \mathbb{R}^K$ be a vector-valued random process. At time $t = 0$, define $V_\kappa(0) \sim Q$ under the product coupling (i.e. independently sampled) and $N(t) = 0 \in  \mathbb{N}^K$

We define $V_\kappa(t)$ for $t >  0 $ recursively as follows. $Z(t) \sim Q$ is an i.i.d sample from the distribution. Let $i(t) = \argmin_{\kappa \in \{1,...,k\}}|V_\kappa(t) - Z(t)|$. Then, $N(t) = N(t-1 )+ \mathbf{1}_{i(t)}$ where $\mathbf{1}_{i(t)}$ is the one-hot vector at index $i(t)$ and 

\[ V_j(t) = \begin{cases} 
    \frac{1}{N_j(t)} \left( N_j(t-1) V_j(t-1) + Z(t) \right) & j = i(t)    \\
    V_j(t-1) &  j \neq  i(t)
   \end{cases}
\]

\end{definition}

For sequential $K$-means processes, the following holds \cite{macqueen1967some}:

\begin{lemma} (J. MacQueen, 1967)

Let $V \in \mathbb{R}^d$ be a a sequential $k$-means process. Then $V$ converges a.s. to an \emph{unbiased partition} jointly satisfying: 

$$\textbf{(i)} V_j(\infty) = \argmin_{v \in \mathcal{S}_j}    \int_{z \in \mathcal{S}_j} (v - z)^2 dQ(z) $$

and 

$$\textbf{(ii)   } \mathcal{S}_j = \{z : j = \argmin_{[K]} | V_j(\infty) - z|  \} \subset \mathbb{R} $$

Equivalently, these conditions state that $V_j$ is the centroid of its Voronoi interval \cite{fortune1992voronoi, aurenhammer2000voronoi} with respect to the vector $V$.

\end{lemma}

We proceed to give a reduction from the competitive learning market with OLS predictors with a particular choice of $P_{XY}$. In turn, this will enable us to concretely characterize the weight estimates asymptotically via MacQueen's lemma.

Let $W = 1$. Let $P_X$ be a point mass at $1$. Let $\epsilon_i \sim \mathbf{Unif}(-\delta,\delta)$ for $0 < \delta$. In this construction, $P_Y \sim  \mathbf{Unif}(1-\delta,1+\delta)$ since $Y = 1 + \epsilon$. In the scalar case, the OLS rule is $\hat{W} = \frac{X^n \bullet Y^n}{||X^n||^2}$. For the assumed distribution, the weight estimate is determined by the empirical mean: 

\begin{equation}
    \hat{W} = \frac{\sum_i y_i}{n} = 1 + \Bar{\epsilon} = \hat Y
\end{equation}

Let $\hat{W}_t(i)$ denote the $i$-th predictor's weight estimate at round $t$, with $\hat{W}_0(i)$ being the initial estimate based on the the seed samples. Let $\nu_t(a) = s + \sum_{j=0}^t \mathbf{1}(\mathbf{SELECT}((j) = a) $ denote the total number of observations made by agent $a$ by round $t$ (where $\mathbf{1}()$ is the indicator function ). When there is no ambiguity, we will write $a_t = \mathbf{SELECT}(t)$ as shorthand.

At round $t$, the OLS update is given by:

\begin{equation}
    \hat{W}_t(a_t) \gets \frac{\nu_t(i) - 1}{\nu_t(i)} \hat{W}_{t-1}(a_t) + \frac{1}{\nu_t(i)} y_t
\end{equation}

Under the assumption of perfect information and rational consumers ($\alpha = \infty$) we have that:

\begin{equation}
    a_t = \argmin_{a \in \mathcal{A}} |\hat{W}_t(a) - y_t|
\end{equation}

In prose, at time $t$, the sample $Y_t$ is averaged into the closest of the OLS estimates at time $t$. This is precisely the update rule use in \cite{macqueen1967some} to describe the sequential $K$-means. However, while the update rule is the same, there remains one blemish that we must smooth over before we can apply MacQueen's lemma. Namely, that the intializations between our case and the sequential $K$-means are not the same, since here the seed dataset can consist of more than one example. It turns out, however, that this issue is easily remedied by noting that the tail of any sequential $K$-means process is conditionally independent of $Y(t)$ given only $V(t)$ and $N(t)$. In other words, for $\tau >0$:

\begin{equation}
    V(t+\tau) \indep (Y(0),..,Y(t)) | (V(t), N(t))
\end{equation}

Therefore, when we have $s$ seed points per $k$ agents, we may construct an equivalence relation between the our process at time $0$ and the sequential $K$-means process at time $(s-1)k$ via:

\begin{equation}
\forall i :   N_i((s-1)k) \gets  \nu_i(0) = s 
\end{equation} and
\begin{equation}
    V((s-1)k) \gets  \hat W(0)
\end{equation}

Thus, we can treat our process $\hat W_t$ as a sequential $K$-means process conditioned on the sequence head as just stated. From the definition of the sequential $K$-means process, it is clear that conditioning on this measure zero event remains well-posed. Furthermore, we can conclude that it does not alter the convergence of the tail because for any $I \subset \mathbb{R}^K$ we have that 

\begin{equation}
    \mathbf{Pr}(\hat W(0) \in I) > 0 
\end{equation}
     if and only if 
\begin{equation}
        \mathbf{Pr}( \{ V((s-1)k) \in I \}\cap \{N_i((s-1)k) = s, \forall i\}) > 0
\end{equation}
  
  which establishes absolute continuity \cite{billingsley2008probability} of the corresponding probability measures over the possible initialization. From this follows the a.s. convergence of $\hat W$ \cite{kallenberg2017random}, and we may dispense with the issue of the initialization and conclude that $\hat Y$  converge a.s. to an unbiased partition (recalling that $\hat Y = \hat W$ for our chosen $P_{XY}$).

It remains to make use of MacQueen's lemma to finish our claim. To do so, we point out the well-known result that there is only one unbiased partition of the uniform distribution over an interval $(1-\delta,1+\delta)$ \cite{Lloyd1982LeastSQ}, namely, the uniform quantization that sets:

\begin{equation}
    \hat Y_i(\infty) = \frac{i2\delta}{k+1} - \delta + 1
\end{equation}

or some permutation thereof. Also notice that this result is translation invariant. Translating the interval corresponds to a translation of the quanta. 

It remains to compute the expected MSE over the $Y_i(\infty)$, noting that \emph{a priori}, each predictor is uniformly likely to converge to any of the quanta.

\begin{equation}
    \mathbf{E}[(Y - \hat Y(\infty))^2] = \sum_{i=1}^n \mathbf{Pr}[ \hat Y(\infty) = \frac{i2\delta}{k+1} - \delta + 1] \mathbf{E}[(Y -\frac{i2\delta}{k+1} + \delta -1)^2]
\end{equation}

\begin{equation}
 = \sum_{i=1}^k \mathbf{Pr}[ \hat Y(\infty) = \frac{i2\delta}{k+1} - \delta] \int_{-\delta}^{\delta}  \frac{1}{2\delta} (\delta -\frac{i2\delta}{k+1} + y)^2 dy
\end{equation}

\begin{equation}
   = \frac{1}{k}\sum_{i=1}^{k}  \frac{1}{2\delta} \int_{-\delta}^{\delta} (\delta -\frac{i2\delta}{k+1} + y)^2 dy
\end{equation}

\begin{equation}
   = \frac{1}{k}\sum_{i=1}^{k}  \frac{1}{2\delta} \left(  \frac{8\delta^3(3i^2 + 3i(k+1) + (k+1)^2)}{3(k+1)^2}  \right)
\end{equation}

\begin{equation}
   = \frac{2\delta^2 k}{3(k+1)}
\end{equation}

To complete the proof, we point out that $W^* = 1$ and that the MMSE for predicting $Y$ given $X$ is given by:

\begin{equation}
    \mathbf{MSE^*}(P_{Y|X}) = \frac{\delta^2}{3}
\end{equation}

And that we may lower bound the $\sup_{P} \frac{\mathcal{R}^k_{t}}{\mathcal{R}^1_{t}}$ with our particular choice of $P_{XY}$.

Thus, taking the limit:

\begin{equation}
  \lim_{T \rightarrow \infty}  \sup_{P} \frac{\mathcal{R}^k_{T}}{\mathcal{R}^1_T} \geq   \lim_{T \rightarrow \infty} \frac{\mathbf{E}[(Y - \hat Y(T)^2]}{\mathcal{R}^1_T} = \frac{\mathbf{E}[(Y - \hat Y(\infty)^2]}{\mathbf{MSE^*}(P_{Y|X})} = \frac{2k}{(k+1)} 
\end{equation}

This completes the proof of part (ii). 

\paragraph{Proof of (i)} We now turn out attention to the proof of part (i). It will suffice to construct a different distribution that can be easily analyzed. Let $X = \frac{1}{2}\Delta$ and define $\epsilon$ as follows:

\[ \epsilon = \begin{cases} 
    \frac{1}{2}\Delta & \text{with prob.} \frac{1}{2}  \\
    -\frac{1}{2}\Delta &   \text{with prob.} \frac{1}{2}\\
   \end{cases}
\]

Thus, we have that $Y = 0$ with prob. $1/2$ and $Y=\Delta$ with prob. $1/2$. The MMSE estimate $W^* = 1$ and the corresponding MSE is $\frac{1}{4}\Delta^2$. Let $\mu_i = \frac{1}{\text{\# samples for } i} \sum_{j} y_j$ denote the empirical mean of the samples observed by agent $i$. The OLS estimate will satisfy $\hat{Y}_i = \frac{1}{2}\Delta \hat W_i = \mu_i$. Due to the OLS-specified bijection between $\hat{Y}$ and $\hat{W}$ we will work with $\hat Y$ for convenience without loss of precision. 

Recall we have $n$ predictors. At round $t = 0$, each has been seeded with $s$ i.i.d samples. Let $\mu_i$ be the empirical mean of the seed samples for the $i$-th predictor. Recall that OLS estimate $\hat{Y}_i = \mu_i$.

Let $i^+ = \argmax\{ \hat{Y}_i\}$ and $i^- = \argmin\{ \hat{Y}_i\}$. For now, we will assume that these extreme predictors are unique, and later address the case in which they are not.

Let $\hat{a}_t = \mathbf{SELECT}(t)$. In the limit as $\Delta \rightarrow \infty$ we have that \emph{only} the two extreme predictors, $a_{i^+}$ and $a_{i^-}$ will ever win consumer queries: $\mathbf{Pr}( \hat{a}_t =  a_{i^+} | Y = \Delta) \rightarrow 1$ and $\mathbf{Pr}( \hat{a}_t =  a_{i^-} | Y = 0)\rightarrow 1$ as $\Delta \rightarrow \infty$. To see this, observe that: 

\begin{equation}
 \frac{\mathbf{Pr}( \hat{a}_t =  a_{i^+} | Y = \Delta)}{\mathbf{Pr}( \hat{a}_t =  a_{j} | Y = \Delta)} =  \frac{ \exp{ \left (\alpha(\Delta - \mu_{i^+})^2\right)}}  {\exp \left(\alpha(\Delta - \mu_j)^2\right)}  = \exp \left( \alpha (\mu_{i^+}^2 - \mu_j ^ 2  + 2\Delta(\mu_{i^+} - \mu_j))   \right)
\end{equation}
And so, when $j \neq i^{+}$ and $\alpha > 0$:

\begin{equation}
   \lim_{\Delta \rightarrow \infty}  \frac{\mathbf{Pr}( \hat{a}_t =  a_{i^+} | Y = \Delta)}{\mathbf{Pr}( \hat{a}_t =  a_{j} | Y = \Delta)} = \infty 
\end{equation}

because of the fact that $\mu_{i^+} \geq \mu_{j}$. In the case when the inequality is strict,  we can immediately conclude that: 

\begin{equation}
    \lim_{\Delta \rightarrow \infty}  {\mathbf{Pr}( \hat{a}_t =  a_{i^+} | Y = \Delta)} = 1
\end{equation}

and more precisely, for some $\varphi > 0$

\begin{equation}
   \mathbf{Pr}( \hat{a}_t =  a_{i^+} | Y = \Delta) = 1 - \Theta(e^{-\varphi  \Delta})
\end{equation}

Due to the  fact that number outcomes (i.e. the consumer decision) is finite and a probability vector is normalized. If the $\argmax$ is not unique, then in the first round at which $Y = \Delta$, one of the maximizing predictors will be selected arbitrarily, which will break the equality. Along those lines, note that $\hat W_{i^+}$ is monotone increasing after observing more samples of $Y = \Delta$. A similar argument holds in the case when $Y = 0$.

Combining these two cases, we conclude that only the $2$ extreme predictors will ever receive additional samples, whereas the other $n-2$ predictors will maintain the weight estimates based solely on their seed samples.

Later on, it will be useful to let $\Delta$ scale with the number of rounds $T$. This does not pose a difficulty when the scaling is chosen judiciously. With $\Delta = \sqrt{T}$, it is easy to verify that the aforementioned limits hold over all rounds. Let $\mathcal{E}(\tau)$ be the event that $\hat{a}_t =  a_{i^+}$ for all $1 \leq t < \tau$ with $Y_t = \Delta$. Then:

\begin{equation}
    \mathbf{Pr}( \mathcal{E}(T)) \geq  \prod_{1 \leq t < T} \mathbf{Pr}( \hat{a}_t =  a_{i^+} | Y_t = \Delta, \mathcal{E}(t-1)) \geq  \prod_{1 \leq t < T} \mathbf{Pr}( \hat{a}_0 =  a_{i^+} | Y_0 = \Delta)
\end{equation}

because of the conditional independence of consumer decisions and monotone increasing trajectory of $\hat W_{i^+}$ in time. However, recalling the convergence rate in $\Delta$ previously established:

\begin{equation}
  \mathbf{Pr}( \mathcal{E}) = \Theta (1-e^{-\varphi \sqrt{T}})^T
\end{equation}

And thus, with the superpolynomial convergence we still obtain:

\begin{equation}
    \lim_{T \rightarrow \infty} \mathbf{Pr}( \mathcal{E}) = 1
\end{equation}
    
We will return to this analysis later to complete the proof. For now, we can proceed to lower bound the gap by taking a weighted average of the MSE in between the extremal and non-extremal predictors. For the two extreme predictors, we will have that $\hat{Y} \rightarrow \Delta$ and $\hat{Y} \rightarrow 0$. In this case the MSE is $\frac{1}{2}\Delta^2$ for both.

For the remaining non-extremal predictors, we can bound the expected MSE as follows. Since they do not obtain further samples beyond the seed set, we can directly analyze the expected MSE of the OLS estimate under $s$ seed samples. However, we must still account for the fact that we are conditioning on the event that these are non-extremal estimates. This is rather cumbersone and is difficult to do exactly, but we can use the following 3 steps of approximations to obtain a lower bound.

(i) The conditional variance of a non-extremal estimate is most reduced when $n=3$.

\begin{equation}
    \hat{Y}_{i^+} = \mathbf{max} \{  \hat Y_j \}_{j=1}^{n}
\end{equation}

\begin{equation}
    \hat{Y}_{i^-} = \mathbf{min} \{  \hat Y_j \}_{j=1}^{n}
\end{equation}

Thus, we have that $\hat Y_{j}  \in [\hat{Y}_{i^-} ,\hat{Y}_{i^+}]$ for all $j$. Recall we are seeking a lower bound on 

\begin{equation}
    \mathbf{Var}(\hat Y_j | j \neq i^+, j \neq i^- ) = \mathbf{Var}(\hat Y_j | \hat Y_{j}  \in [\hat{Y}_{i^-} ,\hat{Y}_{i^+}])
\end{equation}

From this, we immediately obtain that for $n'' \geq n'$:

\begin{equation}
    \mathbf{Var}(\hat Y_j | j \neq i^+, j \neq i^- , n=n') \leq \mathbf{Var}(\hat Y_j | j \neq i^+, j \neq i^- , n=n'')
\end{equation}

because the cdf of $\hat{Y}_{i^+}$ ($\hat{Y}_{i^-}$) is monotone increasing (decreasing) in $n$, given that it is the maximum (minimum) of $n$ i.i.d random variables. Therefore, lowering bounding the case when $n=3$ is sufficient to lower bound all cases.


(ii) We may lower bound the probability of a lower tail deviations of $\hat Y_{i^-}$ by  ignoring that it is the minimum -- i.e. treating it as any generic i.i.d sample -- and then by using the standard bound \cite{matouvsek2001probabilistic}:

$\mathbf{Pr}(\mathbf{Bi}(s,\frac{1}{2}) < k) \geq \frac{1}{15} \exp\left(\frac{-16}{k}(k/2 -k)^2 \right) $  

(Recalling that $\hat Y \sim \frac{\Delta}{2s} \mathbf{Bi}(s,\frac{1}{2})$ ).

Setting a deviation of $\frac{\sqrt{s}}{4}$ results in a lower bound of $\frac{1}{15e}$ on the lower tail. We can use symmetry to apply the argument to  $\hat Y_{i^+}$ as a lower bound to the upper tail. Thus, we have that 

\begin{equation}
    \mathbf{Pr}( \frac{\Delta}{2} +
\frac{\Delta \sqrt{s}}{4}  \leq \hat Y_{i^-} \leq  \frac{\Delta}{2} - \frac{\Delta \sqrt{s}}{4} ) \geq \frac{1}{1764}
\end{equation}

(iii) Once an interval has been established, we can treat the non-extremal estimate $\hat W_{j}$ as a truncated Binomial. By Lemma \ref{lemma:binvar}, we can lower bound the variance of the truncated Binomial with a Binomial over the truncated support. If $Z \sim \mathbf{Bi}(s,\frac{1}{2})$ then

\begin{equation}
\mathbf{Var}\left( Z|\{\frac{\sqrt{s}}{4} \leq Z - \frac{s}{2} \leq \frac{\sqrt{s}}{4} \}\right) \geq \frac{\sqrt{s}}{8}
\end{equation}

and 

\begin{equation}
\mathbf{Var}\left(\frac{\Delta}{2s} Z|\{\frac{\sqrt{s}}{4} \leq Z - \frac{s}{2} \leq \frac{\sqrt{s}}{4} \}\right) \geq \frac{\Delta}{16s^{3/2}}
\end{equation}

Recalling that $\hat Y_j \sim \frac{\Delta}{2s}\mathbf{Bi}(s,\frac{1}{2})$, we have the variance conditioned on event $\mathcal{H} = \{j \neq i^+  , j \neq i^-\}$ is lower bounded as follows: 

\begin{equation}
    \mathbf{Var}(Y_j| \mathcal{H} ) = \mathbf{Var}(Y_j|\mathcal{H}, \mathcal{F})\mathbf{Pr}(\mathcal{F}) +  \mathbf{Var}(Y_j| \mathcal{H} , \bar{ \mathcal{F}})(1-\mathbf{Pr}(\mathcal{F})) \geq \mathbf{Var}(Y_j|\mathcal{H}, \mathcal{F})\mathbf{Pr}(\mathcal{F}) 
\end{equation}

And since we know have a bound for $\mathbf{Pr}(\mathcal{F})$ and $\mathbf{Var}(Y_j|\mathcal{H},\mathcal{F})$ we conclude:

\begin{equation}
   \mathbf{Var}(Y_j| \mathcal{H} ) \geq  \mathbf{Var}(Y_j|\mathcal{H}, \mathcal{F})\mathbf{Pr}(\mathcal{F})  \geq \frac{\Delta^2}{(7056 \times 4) s^{3/2}}
\end{equation}

Where $\mathcal{F} =  \{ \hat Y_{i^-} \leq  \frac{\Delta}{2} - \frac{\Delta \sqrt{s}}{4}, \hat Y_{i^-} \geq \frac{\Delta}{2} +
\frac{\Delta \sqrt{s}}{4}\} $

Note that $\hat{Y}_j$ is unbiased. From this, it follows that the excess MSE risk is given by its variance. Define $\mathcal{\delta} =  \hat{Y}_j - \mathbf{E}\hat{Y}_j = \hat{Y}_j - \frac{\Delta}{2}$.

$$ \mathbf{MSE}(\hat Y_j) = \mathbf{E}[ \frac{1}{2}( \frac{\Delta}{2} - \delta)^2 + \frac{1}{2}(\frac{\Delta}{2} +\delta)^2]  = \frac{\Delta^2}{4} + \mathbf{E}\delta^2 = \mathbf{MSE}^* + \mathbf{Var}(\hat Y_j) $$

Note that we have established that for the extremal predictors, the asymptotic MSE is $\frac{\Delta^2}{2}$, and for the non-extremal predictors, the asymptotic MSE is at least $\frac{\Delta^2}{4} + \frac{\Delta^2}{(7056 \times 4) s^{3/2}}$. Combining these in-expectation gives: $\frac{k-2}{k}(\frac{\Delta^2}{4} + \frac{\Delta^2}{(7056 \times 4) s^{3/2}}) + \frac{2}{k}\frac{\Delta^2}{2}$

Thus, to conclude the proof, let us scale $\Delta = \sqrt T$:

$$  \lim_{T \rightarrow \infty}  \sup_{P} \frac{\mathcal{R}^k_{T}}{\mathcal{R}^1_T} \ \geq \lim_{T \rightarrow \infty} \frac{\frac{k-2}{k}(\frac{\Delta^2}{4} + \frac{\Delta^2}{(7056 \times 4) s^{3/2}}) + \frac{2}{k}\frac{\Delta^2}{2}}{\frac{\Delta^2}{4}} \geq 1 + \frac{1}{7056s^{3/2}} + \frac{2}{k}$$.

which yields the inequality in part (ii).

\end{proof}

\subsubsection{Proof of Theorem 4.4.}

As a clarification, when we refer to pairwise covariance $\rho$, we refer to:

\begin{equation}
    \rho = \mathbf{Cov}(W_1,W_2) = \mathbf{E}(W_1W_2) - \mathbf{E}(W_1)\mathbf{E}(W_1)
\end{equation}

where $W_1 = \mathbf{1}(\hat Y^1 = Y)$ and $W_2 = \mathbf{1}(\hat Y^2 = Y)$ for $\hat Y^1$ and $\hat Y^2$ denote the predictions from agent $1$ and agent $2$ respectively. Notice that $\mathbf{E}(W_1W_2)$ equals the probability that both predictions are correct.  At any given time $t$, the quantity $\mathbf{E}(W_1W_2)$ itself is stochastic, given that it depends on the particular sample path taken by the random competition up until the given time. In our proof, we will use $P(t)$ to denote the measure over predictor correctness at a given time $t$

In order to proceed, we must first recall the definition of $\mathbb{A}_\tau$ from the main text: $ \mathbb{A}_\tau = \mathbf{E}(\mathbf{1}\{ \hat{Y}_\tau ^{(w_\tau)} = Y_\tau\})$. Note that this expectation implicitly takes place over the entire randomness in the learning competition $\mathcal G$ (refer to Def. C.1).
For the purposes of this Theorem, from a formal perspective, we will be comparing the the quantities $\mathbb{A}_\tau$ corresponding to different instances of $\{ \mathcal G_\kappa \}_{\kappa=1}^\infty$ that vary only in the number of predictors. We write $\mathbb{A}_\tau^k$ to refer to the expectation over the randomness of learning competition $\mathcal{G}_k$ with $k$ predictors:

$$ \mathbb{A}_\tau^k = \mathbf{E}_{\mathcal{G}_k}(\mathbf{1}\{ \hat{Y}_\tau ^{(w_\tau)} = Y_\tau\})$$

\begin{theorem}

Assume a learning competition at round $\tau$. Define $\mathcal{A}(A^{(i)};\mathcal{D}) = 1 - \mathcal{R}(A^{(i)};\mathcal{D})$ and $\mathcal{A}_k(t) = 1 - \mathcal{R}_t^k$. When the parameter $t$ is omitted, assume $t = \tau$: $\mathcal{A}_k = \mathcal{A}_k(\tau)$. Define $\delta = \mathcal{A}_1 - \mathcal{A}_2$ and $\varepsilon = \mathcal{A}_k(0) - \frac{1}{2} $. Let $\rho$ be the pairwise covariance between two predictors. Assume the following holds:

\begin{enumerate}
    \item There is sufficient data to train two predictors well: $0 < \delta < \frac{1}{6}$
    \item  Predictors are weak with only seed data: $\varepsilon < 1/14$
    \item  The predictors are not too correlated: $  \rho < \mathcal{A}_k - \mathcal{A}_k^2 - 6 \delta$
    \item The expected accuracy of a predictor monotonically increases with dataset size: $\mathbf{E}_{\mathcal{G}}\left[\mathcal{A}(A^{(i)}_t;\mathcal{D}) \vert r = |D_t^{(i)}|\right]$
\end{enumerate}

Then there exists $0 < c_1 < c_2 < \infty$ such that if $c_1 < \alpha < c_2$ then $\mathbb{A}_\tau^k$ at round $\tau$ is maximized by some $k^*$ number of predictors such that $1 < k^* < \infty$. In particular, $c_1 < \log \frac{\mathcal{A}_1 -(\mathcal{A}_1 - \delta)^2 - \rho}{\mathcal{A}_1 -(\mathcal{A}_1 - \delta)^2 - \rho - 2\delta}$ and $c_2 >  \log \frac{(1-4\varepsilon)\mathcal{A}_1}{1- \mathcal{A}_1}$.


\end{theorem}{}

\begin{proof}

For notational convenience, we slightly modify the notation for expected accuracy $\mathcal{A}_k(t)$ compared with the notation  $\mathcal{A}_t^k$ introduced in the main text. We have move the time index $t$ from a subscript into a parenthetical function argument and moved the number of predictors $k$ from a superscript to a subscript.

By assumption, we have that predictors at round $0$ are weak predictors, meaning they are independently accurate with probability $\mathcal{A}_k(0) = \frac{1}{2} + \varepsilon$. Note we thus have $\varepsilon < 1/14$ by the assumption.  As we will see shortly, in the limit of infinite predictors, the algorithm's performance after obtaining additional samples is immaterial.



As defined in the main text, let denote $\mathbb{A}_\tau^k$ the expected prediction quality for users at time $\tau$ of $\mathcal{G}_k$. In other words,  $\mathbb{A}_\tau^k$ is the expected prediction quality for users when the competition has $k$ learners. Let $\hat {Y}_{\tau}^{(i)}$ be the $i$-th learner's prediction at time $\tau$ (or equivalently let $\hat {Y}_{\tau}^{(a)}$ be learner $a$'s prediction). When it is unambiguous to do so, we will define $w_\tau = \mathbf{SELECT}(\mathbf{q}_\tau)$ as described in Section 2 of the main text. Because most of the variables in this proof are implicit at time $\tau$, when the time is not explicitly stated or sub-scripted, assume that the variable refers to time $\tau$, henceforth.



\begin{equation}
    \mathbb{A}_\tau^k =  \sum_{a \in \mathsf{A}}   \mathbf{Pr}[a = w_\tau] \mathbf{Pr}[\hat {Y}_{\tau}^{(a)} = Y_\tau | a = w_\tau]
    \label{eqn:A}
\end{equation}
where we recall from the definition of $\mathbf{SELECT}$:
\begin{equation}
  \mathbf{Pr}[a = w_\tau]  = \frac{\exp(\alpha\mathbf{1}\{\hat {Y}_{\tau}^{(a)} = Y_\tau \})}{ Z_\tau}
 \end{equation}
 where $Z_\tau = \sum_{a \in \mathcal{A}} \exp(\alpha\mathbf{1}\{ \hat {Y}_{\tau}^{(a)}  = Y\})$.

We proceed to give an expression for  $\mathbf{Pr}[a = a_t]$ in the limit as $k \rightarrow \infty$. Let $\mathsf{B}_\tau \subset \mathsf{A}$ be the subset of predictors that have been queried at least once by some time $\tau$ (subscript on the $\mathsf{B}$ is omitted when it may be safely inferred). It follows that $\frac{|\mathsf{B}|}{|\mathsf{A}|} \leq \frac{\tau}{k}$ which implies $\frac{|\mathsf{B}|}{|\mathsf{A}|} \rightarrow 0 $ as $k \rightarrow \infty$. From this, it follows that the $\mathbf{Pr}[ a \in \mathsf{B}] \leq \frac{\tau e^\alpha}{k}$ which also vanishes as $k$ gets large. With this, we can revise Eq. \ref{eqn:A} as follows:

\begin{equation}
     \mathbb{A}_\tau^k =  \sum_{a \in \mathsf{B}}   \mathbf{Pr}[a = w_\tau| a \in \mathsf{B} ] \mathbf{Pr}(\hat Y^a = Y| a  = w_\tau ) + \sum_{a \notin \mathsf{B}}   \mathbf{Pr}[a = w_\tau| a \notin \mathsf{B} ] \mathbf{Pr}(\hat Y^a = Y| a \notin \mathsf{B})
\end{equation}

Here, notice that if $w_\tau \notin \mathsf{B}$, then $\mathbf{Pr}(\hat Y^a = Y| a \notin \mathsf{B}) = \mathbf{Pr}(\hat Y^a = Y| a =w_\tau)$ because all predictors that have not been selected are modeled as weak learners and are thus interchangeable.

We give the following  lower and upper bounds for  $\mathbb{A}$, which can be easily derived using the law of total probability \cite{grimmett2014probability} and the fact that $0 \leq \mathbb{A} \leq 1$:

\begin{equation}
\mathbf{Pr}(w_t \notin \mathsf{B}_t)\left(\mathbb{A}_t^k|\{w_t \notin \mathsf{B}_t \}\right)  \leq  \mathbb{A}_t^k \leq \mathbf{Pr}(w_t \notin \mathsf{B}_t)\left(\mathbb{A}_t^k|\{w_t \notin \mathsf{B}_t \}\right) + \mathbf{Pr}[w_t \in \mathsf{B}_t]
 \end{equation}
 
Where because $\mathbb{A}_t^k$ is already an expectation, we use the notation $\mathbb{A}_t^k|\{a_t \notin B_t \}$ to mean the conditional expectation.

Taking the limit in $k$  vanishes $\mathbf{Pr}[a \in \mathsf{B}_t]$, yielding:  

\begin{equation}
\mathbb{A}_t^\infty = \mathbb{A}_t^\infty |\{w_t \notin \mathsf{B}_t \}
\end{equation}

We proceed to compute this quantity $\mathbb{A}^\infty |\{a_t \notin \mathsf{B} \}$. To do so, we make use of the weak predictor assumption. From that, may treat the aggregate predictions from predictors not in $\mathsf{B}_t$ as following a Binomial distribution with success probability  of $ \frac{1}{2} + \varepsilon$.  Define $\kappa_t = |\mathsf{A} / \mathsf{B}_t | \in [k-t, k-1] = \Theta(k)$. Let $V_t$ be the number of weak predictors with correct predictions at time $t$. Then $V \sim \mathbf{Bin}(\kappa_t, \frac{1}{2} + \varepsilon)$. Let $\mu$ denote the mean of $V$, $\mathbf{E}V = \mu$.  Then, the conditional probability that the consumer at $t$ selects any correct weak predictor is:

\begin{equation}
    \mathbf{Pr}(\hat Y^{w_t} = Y | w_t \notin \mathsf{B}_t, V = \nu) = \frac{\nu e^{\alpha}}{\nu (e^{\alpha}-1) + \kappa}
\end{equation}

Let us rewrite $V$ in terms of its deviation from its mean: $V = \mu + \Delta$ for implicitly defined random deviation $\Delta$. By the central limit theorem  \cite{grimmett2014probability,billingsley2008probability,klenke2013probability,rosenblatt1956central}, we know that deviations of the Binomial with $\kappa$ trials \cite{grimmett2014probability,billingsley2008probability,klenke2013probability} are order $\Theta(\sqrt{k})$ implying $\mathbf{Pr}(|\Delta| \geq Ck^{0.6}) \rightarrow 0$ as $k \rightarrow \infty$ for any finite $C$. Thus, let us introduce the deviation into the above equation and normalize by $\kappa$:

\begin{equation}
\mathbf{Pr}(\hat Y^{w_t} = Y | w_t \notin \mathsf{B}_t, \Delta = \mathbf{\Delta}) = \frac{ (\frac{\mu}{\kappa} + \frac{\mathbf{\Delta}}{\kappa}) e^{\alpha}  }{(\frac{\mu}{\kappa} +\frac{\mathbf{\Delta}}{\kappa} ) (e^{\alpha}-1) + 1} = \frac{ (\frac{1}{2} + \varepsilon + \frac{\mathbf{\Delta}}{\kappa}) e^{\alpha}  }{(\frac{1}{2} + \varepsilon +\frac{\mathbf{\Delta}}{\kappa} ) (e^{\alpha}-1) + 1}
\end{equation}

By the aforementioned line of reasoning, because the the risk is bounded, we may ignore any large deviations for $\Delta$ with respect the mean when taking a limit in $k$. Computing the limit is direct and yields:

\begin{equation}
    \lim_{k \rightarrow \infty}  \frac{ (\frac{1}{2} + \varepsilon + \frac{\mathbf{\Delta}}{\kappa}) e^{\alpha}  }{(\frac{1}{2} + \varepsilon +\frac{\mathbf{\Delta}}{\kappa} ) (e^{\alpha}-1) + 1} = \frac{ (\frac{1}{2} + \varepsilon ) e^{\alpha}  }{(\frac{1}{2} + \varepsilon ) (e^{\alpha}-1) + 1} =  \frac{e^{\alpha}  }{e^{\alpha} + 1 - \chi}
\end{equation}

For $\chi = \frac{4\varepsilon}{2\varepsilon + 1}$. Notice that $\varepsilon < \frac{1}{6}$ implies $\chi < \frac{1}{2}$. From this we establish: $\mathbb{A}^\infty =\frac{e^{\alpha}}{e^{\alpha} + 1 -\chi }$. By definition, we know $\mathbb{A}^1 = \mathcal{A}_1$.

Given that the joint distribution for 2 predictors is determined by the marginals and the covariance, we can solve for $\mathbb{A}^2$. We introduce $P$, the probability measure over the correctness of each of the two predictors at a fixed time $t$ (the dependence from $P$ on $t$ is only stated here and omitted in notation). $P_{11}$ denotes the probability that both predictors are correct and $P_{00}$ denotes the probability that both are incorrect. $P_{10}$ denotes the probability that the first predictor is correct and the second is incorrect and $P_{01}$ denotes the probability of the opposite. There are the four possible outcomes under measure $P$ for only two predictors ($k=2$). Of course, $P$ itself is a randomized object, given that it depends on the randomness generated by the competition $\mathcal{G}$ until time $t$.

We find $ \mathbb{A}_t^2 $ in terms of $P$:

\begin{equation}
    \mathbb{A}_t^2 = \mathbf{E}_{P \sim \mathcal{G}_2}(P_{11} + \frac{e^{\alpha}}{e^\alpha + 1}(P_{10} + P_{01}) )
\end{equation}

where the above follows directly from the structure of the competition. 

From the definition of covariance we know that:

\begin{equation}
    \rho = P_{11} - (P_{11} + P_{10})(P_{11}+P_{01})
\end{equation}

Substituting for $P_{11}$ yields:

\begin{equation}
    \mathbb{A}_t^2 = \mathbf{E}_{P \sim \mathcal{G}_2}\left[\rho + (P_{11} + P_{10})(P_{11} + P_{01}) + \frac{e^{\alpha}}{e^\alpha + 1}(P_{10} + P_{01}) \right]
\end{equation}

Furthermore (we drop the explicit expectation over $\mathcal{G}_2$):

\begin{equation}
    \mathbb{A}_t^2 = \mathbf{E}[\rho] + \mathbf{E}[(P_{11} + P_{10})(P_{11} + P_{01})] + \frac{e^{\alpha}}{e^\alpha + 1}\mathbf{E}(P_{10} + P_{01}) 
\end{equation}

And notice that $P_{11} + P_{10}$ is the marginal probability that the first predictor is correct. Thus, $P_{11} + P_{10} = \mathcal{A}(A^{(1)};\mathcal{D})$ and $P_{01} + P_{11} = \mathcal{A}(A^{(2)};\mathcal{D})$ by definition (for predictor $i$ we may omit the parenthesis in the superscript, i.e. we write $A^i$ instead of $A^{(i)}$). 

\begin{equation}
    \mathbb{A}_t^2 = \mathbf{E}[\rho] + \mathbf{E}[\mathcal{A}(A^1;\mathcal{D})\mathcal{A}(A^2;\mathcal{D})] + \frac{e^{\alpha}}{e^\alpha + 1}\mathbf{E}(P_{10} + P_{01}) 
\end{equation}

where $\sim$ denotes equality of distribution.

Below, we break up the expectation of $(P_{10} + P_{01})$ into two terms and  we add  $P_{11} - P_{11}$ and $(P_{11} + P_{10})(P_{11} + P_{01}) - (P_{11} + P_{10})(P_{11} + P_{01})$ to each term for a net effect of zero.

\begin{equation}
\mathbf{E}(P_{10} + P_{01})  = \mathbf{E}\left[P_{10} + P_{11} - P_{11} +(P_{11} + P_{10})(P_{11} + P_{01}) - (P_{11} + P_{10})(P_{11} + P_{01})\right]
\end{equation}

\begin{equation}
 +  \mathbf{E}\left[P_{01} + P_{11} - P_{11} +(P_{11} + P_{10})(P_{11} + P_{01}) - (P_{11} + P_{10})(P_{11} + P_{01})\right]
\end{equation}

Regrouping the terms produces:

\begin{equation}
\mathbf{E}(P_{10} + P_{01}) =
\mathbf{E}\left[\mathcal{A}(A^1;\mathcal{D}) - \rho - \mathcal{A}(A^1;\mathcal{D})\mathcal{A}(A^2;\mathcal{D})\right] + \mathbf{E}\left[\mathcal{A}(A^2;\mathcal{D}) - \rho - \mathcal{A}(A^1;\mathcal{D})\mathcal{A}(A^2;\mathcal{D})\right]
\end{equation}

As a shorthand, let $\mathfrak{a} = \frac{e^\alpha}{e^\alpha + 1}$.

Substituting it back into the full equation yields:

\begin{equation}
  \mathbb{A}_t^2 = \left( 1 - 2\mathfrak{a} \right)\mathbf{E}[\rho] + \mathfrak{a}\mathbf{E}(\mathcal{A}(A^1;\mathcal{D})) + \mathfrak{a}\mathbf{E}(\mathcal{A}(A^2;\mathcal{D})) +   \left( 1 - 2\mathfrak{a} \right)\mathbf{E}[\mathcal{A}(A^1;\mathcal{D})\mathcal{A}(A^2;\mathcal{D})]
\end{equation}

Notice that we may assume $\mathcal{A}(A^2;\mathcal{D})$ and $\mathcal{A}(A^1;\mathcal{D})$ are \emph{negatively} correlated. This follows from the assumption that $\mathbf{E}_{\mathcal{G}}\left[\mathcal{A}(A^{(i)}_t;\mathcal{D}) \vert r = |D_t^{(i)}|\right]$ is monotone increasing in $r$. Because $|D_t^{(1)}| + |D_t^{(2)}| = t-2s$, as one predictor gets more data, the other must get less (the $-2s$ term comes from accounting for the seed sets). Thus, by assumption:

\begin{equation}
    \mathbf{E}[\mathcal{A}(A^1;D)\mathcal{A}(A^2;D)] \leq  \mathbf{E}[\mathcal{A}(A^1;D)]\mathbf{E}[\mathcal{A}(A^2;D)]
\end{equation}

Because $\mathfrak{a} \geq 1/2$ we know that $(1-2\mathfrak{a}) \leq 0$:

\begin{equation}
     \left( 1 - 2\mathfrak{a} \right)   \mathbf{E}[\mathcal{A}(A^1;D)\mathcal{A}(A^2;D)] \geq (1-2\mathfrak{a})  \mathbf{E}[\mathcal{A}(A^1;D)]\mathbf{E}[\mathcal{A}(A^2;D)]
\end{equation}

Implying,

\begin{equation}
  \mathbb{A}_t^2 \geq \left( 1 - 2\mathfrak{a} \right)\mathbf{E}[\rho] + \mathfrak{a}\mathbf{E}(\mathcal{A}(A^1;\mathcal{D})) + \mathfrak{a}\mathbf{E}(\mathcal{A}(A^2;\mathcal{D})) +   \left( 1 - 2\mathfrak{a} \right)\mathbf{E}[\mathcal{A}(A^1;\mathcal{D})]\mathbf{E}[\mathcal{A}(A^2;\mathcal{D})]
\end{equation}

\begin{equation}
    \mathbb{A}_t^2 = \mathbf{E}[\rho] + \mathbf{E}[\mathcal{A}(A_1;D)\mathcal{A}(A_2;D)] + \frac{e^{\alpha}}{e^\alpha + 1}\mathbf{E}(P_{10} + P_{01}) 
\end{equation}

By the symmetry assumption (i.e. predictors 1 and 2 use the same algorithm), we have that:

\begin{equation}
   \mathcal{A}(A_1;\mathcal{D}) \sim \mathcal{A}(A_2;\mathcal{D})
\end{equation}

Implying,

\begin{equation}
  \mathbb{A}_t^2 \geq \left( 1 - 2\mathfrak{a} \right)\mathbf{E}[\rho] + 2\mathfrak{a}\mathbf{E}(\mathcal{A}(A^2;\mathcal{D})) +   \left( 1 - 2\mathfrak{a} \right)\mathbf{E}[\mathcal{A}(A^2;\mathcal{D})]^2
\end{equation}

Since $\rho$ is given by assumption, we have the simplification that $\mathbf{E}\rho = \rho$. Further, recalling the definition of $\mathcal{A}_2$ yields the simplification that  $\mathbf{E}\mathcal{A}(A^2;\mathcal{D})= \mathcal{A}_2(t)$. Simplifying and rearranging yields:

\begin{equation}
    \mathbb{A}_t^2 \geq \rho + \mathcal{A}_2^2 + \frac{2e^\alpha(\mathcal{A}_2 - \rho -\mathcal{A}_2^2)}{e^\alpha + 1}
\end{equation}

This lower bound on  $\mathbb{A}^2$ will suffice to get us through the rest of the argument. Note that if both $\mathbb{A}^2 > \mathbb{A}^1$ and $\mathbb{A}^2 > \mathbb{A}^\infty$ hold, then the theorem stands. For expediency, we will prove a sufficient condition instead. Namely, we will show that given the assumed initial conditions, we have that $\mathbb{A}^2 > \mathbb{A}^1 > \mathbb{A}^\infty$. The reason for choosing this particular order is not fundamental since the potential order $\mathbb{A}^2 > \mathbb{A}^\infty > \mathbb{A}^1$ would also be sufficient. However, the since the quantity $\mathbb{A}^1$ is trivial, it is simpler to tame the pairwise comparisons $\mathbb{A}^2 > \mathbb{A}^1$ and $\mathbb{A}^1 > \mathbb{A}^\infty$ than the comparison between $\mathbb{A}^2 > \mathbb{A}^\infty$.

Before proceeding, we briefly pause to provide the reader with a bit of intuition. In a sense, we can see that $\mathbb{A}^2$ is larger than $\mathbb{A}^1$ when the boost gained from having the user selection outweighs the penalty due to the competition between the predictors. Thus, we need a \emph{lower bound} on $\alpha$ in order for $\mathbb{A}^2 > \mathbb{A}^1$. On the other hand, if $\alpha$ is too large, then we get too much of a boost from user selection in the $k \rightarrow \infty$ limit making $\mathbb{A}^\infty$ too large. If we want $\mathbb{A}^2 >  \mathbb{A}^\infty$ we also need an \emph{upper bound} on $\alpha$. We proceed to quantify these bounds and show that under the assumed conditions, in between these bounds remains the \emph{sweet spot} interval for $\alpha$ in which both $\mathbb{A}^2 >  \mathbb{A}^\infty$ and $\mathbb{A}^2 > \mathbb{A}^1$ which implies the non-monotonic phenomena.

Let us examine the constraints placed on key quantities $\mathcal{A}_1$, $\rho$, $\delta$, $\varepsilon$, and $\alpha$ based on the two inequalities. Recall $\mathcal{A}_2 = \mathcal{A}_1 - \delta$. Of course, there are an intrinsic set of constraints on these five quantities by assumption or definition:

\begin{equation}
    1 > \mathcal{A}_1 > \frac{2}{3}  
\end{equation}
\begin{equation}
    \frac{1}{6} >  \delta > 0
\end{equation}

\begin{equation}
    \alpha \geq 0
\end{equation}

\begin{equation}
 1 \geq \rho \geq -1
\end{equation}

\begin{equation}
    \frac{1}{4} > \chi > 0
\end{equation}

As discussed previously, there is a 1-to-1 mapping between $\chi$ to $\varepsilon$ given by$\chi = \frac{4\varepsilon}{2\varepsilon + 1}$. This mapping happens to be concave in $\varepsilon$ to $\chi$, so any tangent line to the curve serves as on overestimate from $\varepsilon$ to $\chi$. Noting that $(\chi,\varepsilon) = (\frac{1}{4}, \frac{1}{14})$ lies on the curve lets us conclude that $\chi(\varepsilon) < 4 \varepsilon$. This will later on helpful when transforming bounds into terms of $\varepsilon$.

The question remains what subset of this region remains after imposing that $\mathbb{A}^2 > \mathbb{A}^1 > \mathbb{A}^\infty$. In other words, how should we refine this subset in order to accomplish the desired relationships in $\mathbb{A}$? Let us begin by analyzing $\mathbb{A}^2 > \mathbb{A}^1$:

\begin{equation}
 \rho + (\mathcal{A}_1 - \delta)^2 + \frac{2e^\alpha(\mathcal{A}_1 - \delta - \rho -(\mathcal{A}_1 - \delta)^2)}{e^\alpha + 1} > \mathcal{A}_1
\end{equation}

It is not difficult to check (perhaps by using a computer algebra system such as Mathematica\footnote{For example, using the following one line Mathematica command: \texttt{ Reduce[\{1 > rho > -1, expalpha > 1, 1 > A1 > 2/3, 1/5 > delta > 0, (-delta + A1)\string^2 + (2 expalpha (-delta + A1 - (-delta + A1)\string^2 - rho))/(1 + expalpha) + rho > A1\}, \{expalpha\}]}}) that this inequality is satisfied if the additional constraints hold:

\begin{equation}
   \alpha > \log \frac{\psi - \rho}{\psi - \rho - 2\delta } \text{ and }     \psi - 2\delta > \rho  \text{ where } \psi = \mathcal{A}_1 -(\mathcal{A}_1 - \delta)^2
\end{equation}

Loosely, these conditions say that $\rho$ cannot be too large. This is natural since the users gain no benefit from two learners when they are perfectly correlated. 




We can proceed to look at the second pairing: $\mathbb{A}^1 > \mathbb{A}^\infty$:

\begin{equation}
\mathcal{A}_1 > \frac{e^\alpha}{ e^\alpha + 1 - \chi}
\end{equation}

Elementary algebra yields:

\begin{equation}
\log \frac{(1-\chi)\mathcal{A}_1}{1- \mathcal{A}_1} > \alpha
\end{equation}

It remains to show that this interval $ \log \frac{(1-\chi)\mathcal{A}_1}{1- \mathcal{A}_1} >  \alpha > \log \frac{\psi - \rho}{\psi - \rho - 2\delta } $ is guaranteed to exist.

To do this, note that since $ 0 < \chi < \frac{1}{6}$ and $ 1  > \mathcal{A}_1 > 2/3$ we have that:

\begin{equation}
\frac{(1-\chi)\mathcal{A}_1}{1- \mathcal{A}_1} > \min_{0 < x < \frac{1}{6} , \frac{2}{3} < a < 1} \frac{(1-x)a}{1- a} =  \frac{5}{3}
\end{equation}

Thus we are left to ponder the inequality: $5/3 > \frac{\psi - \rho}{\psi - \rho - 2\delta }$. As before, with the aid of computer algebra\footnote{ For example, using the following one line Mathematica command: \texttt{Reduce[\{A1-A1\string^2 -2delta + 2A1 delta - delta\string^2 > rho >-1 ,1/5>delta>0, 1 > A1 > 2/3, (A1 - A1\string^2 + 2A1 delta - delta\string^2 - rho)/(A1 - A1\string^2 - 2 delta +2 A1 delta - delta\string^2 - rho) <5/3\},\{rho\}]}} we can translate the above inequality into constraints on the quantities of interest. The following additional constraint on $\rho$ is sufficient to imply $\mathbb{A}^2 > \mathbb{A}^\infty$:


\begin{equation}
    \rho < \psi -6\delta
\end{equation}


To conclude, then there exists an non-empty interval for $\alpha$ given by:

\begin{equation}
      \log \frac{(1-\chi)\mathcal{A}_1}{1- \mathcal{A}_1} >  \alpha > \log \frac{\psi - \rho}{\psi - \rho - 2\delta } 
\end{equation}

such that we can guarantee:

\begin{equation}
   1 < k^* < \infty
\end{equation}

if the following also hold:

\begin{equation}
   \mathcal{A}_1 > \frac{2}{3}  
\end{equation}
\begin{equation}
    \frac{1}{6} >  \delta
\end{equation}

\begin{equation}
    \frac{1}{4} > \chi
\end{equation}

\begin{equation}
    \mathcal{A}_2 - \mathcal{A}_2^2 - 6\delta > \rho
\end{equation}

Finally, we translate back to $\varepsilon$ from $\chi$ by nothing that since $\chi < 4\varepsilon$ it follows: $\log \frac{(1-\chi)\mathcal{A}_1}{1- \mathcal{A}_1} > \log \frac{(1-4\varepsilon)\mathcal{A}_1}{1- \mathcal{A}_1}$. Transforming the constraints then yield:

\begin{equation}
    \frac{1}{14} > \varepsilon
\end{equation}
\begin{equation}
      \log \frac{(1-4\varepsilon)\mathcal{A}_1}{1- \mathcal{A}_1} >  \alpha > \log \frac{\psi - \rho}{\psi - \rho - 2\delta } 
\end{equation}

\end{proof}{}

\bibliographystyle{apalike}
\bibliography{competing_ai}

\end{document}